\documentclass{ecai} 

\usepackage{latexsym}
\usepackage{amssymb}
\usepackage{amsmath}
\usepackage{amsthm}
\usepackage{booktabs}
\usepackage{enumitem}
\usepackage{graphicx}
\usepackage{color}

\usepackage{amsmath}
\usepackage{dsfont}
\usepackage{multicol}
\usepackage{multirow} 

\newtheorem{example}{Example}

\usepackage{algorithm}
\usepackage{algpseudocode}

\usepackage{subcaption}

\usepackage{soul}

\newtheorem{theorem}{Theorem}

\newtheorem{remark}{Remark}

\newcommand{\BibTeX}{B\kern-.05em{\sc i\kern-.025em b}\kern-.08em\TeX}

\begin{document}


\begin{frontmatter}


\paperid{123} 


\title{BCR-DRL: Behavior- and Context-aware Reward\\for Deep Reinforcement Learning in Human-AI Coordination}


\author[A]{\fnms{Xin}~\snm{Hao}\orcid{0000-0003-1577-2620}\thanks{Corresponding Author. Email: haoxin1022@hotmail.com}}
\author[A]{\fnms{Bahareh}~\snm{Nakisa}}
\author[B]{\fnms{Mohmmad Naim}~\snm{Rastgoo}} 
\author[C]{\fnms{Gaoyang}~\snm{Pang}} 

\address[A]{Deakin University}
\address[B]{Monash University}
\address[C]{The University of Sydney}


\begin{abstract}
Deep reinforcement Learning (DRL) offers a powerful framework for training AI agents to coordinate with human partners. However, DRL faces two critical challenges in human-AI coordination (HAIC): sparse rewards and unpredictable human behaviors. These challenges significantly limit DRL to identify effective coordination policies, due to its impaired capability of optimizing exploration and exploitation. To address these limitations, we propose an innovative behavior- and context-aware reward (BCR) for DRL, which optimizes exploration and exploitation by leveraging human behaviors and contextual information in HAIC. Our BCR consists of two components: (i)~A novel dual intrinsic rewarding scheme to enhance exploration. This scheme composes an AI self-motivated intrinsic reward and a human-motivated intrinsic reward, which are designed to increase the capture of sparse rewards by a logarithmic-based strategy; and (ii)~A new context-aware weighting mechanism for the designed rewards to improve exploitation. This mechanism helps the AI agent prioritize actions that better coordinate with the human partner by utilizing contextual information that can reflect the evolution of learning. Extensive simulations in the Overcooked environment demonstrate that our approach can increase the cumulative sparse rewards by approximately $20\%$, and improve the sample efficiency by around $38\%$ compared to state-of-the-art baselines.
%
\end{abstract}

\end{frontmatter}


\section{Introduction}
Human-AI coordination (HAIC) has emerged as a critical research area focusing on complicated tasks requiring coordinated behaviors, including the synergy of human intuition and machine autonomy~\cite{NeurIPS2019_Overcooked_GitHub,PantheonRL_aaai2022}. For instance, in the Overcooked environment, an AI agent must coordinate with a human partner to prepare meals by dividing tasks such as chopping ingredients and delivering dishes, requiring adaptive responses to human actions (detailed in Section 3.2). This pressing need demands an AI agent that is adaptive to the human partner in addition to the environment. Deep reinforcement learning (DRL) offers a powerful framework for developing such an adaptive AI agent~\cite{Xin_BCDRL_TCOM,Xin_BCDRL_ICC2024}, leveraging its ability to learn optimal policies through interactions with both the human partner and the environment~\cite{AAAI_overcooked_Tsinghua,IJCAI2024_HRL_Overcooked}. However, employing DRL in HAIC presents two interdependent challenging goals in balancing exploration and exploitation:

    \textit{\textbf{i) Enhancing exploration of critical but rare state-action pairs yielding sparse rewards.}} 
    Rewards are usually sparse in HAIC due to the inherent complexity of coordinated tasks. The state-action pairs associated with sparse rewards in HAIC are critical for effective learning. However, these pairs are rarely encountered during training, since they emerge from coordinated sequences of temporally-extended actions between AI and human agents~\cite{combine_human_aaai_2019}--making them particularly elusive when the AI agent has limited ability to fully interpret the intricate dynamics of unpredictable human behaviors~\cite{multiple_feedback_aaai_2023,balance_human_aaai_2020}. These rare interactions represent key moments when AI actions perfectly align with human intentions, leading to successful task completion. Without sufficient exposure to these critical state-action pairs, the AI agent struggles to understand which actions are beneficial, resulting in slow learning and local optima performance. Existing approaches mitigate this issue by augmenting exploration through intrinsic rewards that supplement those sparse rewards obtained extrinsically from the environment~\cite{ICML2019_Social_MARL_GitHub,max_entropy_reward}.

\begin{figure}[!t]
    \centering
    \includegraphics[width=3.4in]{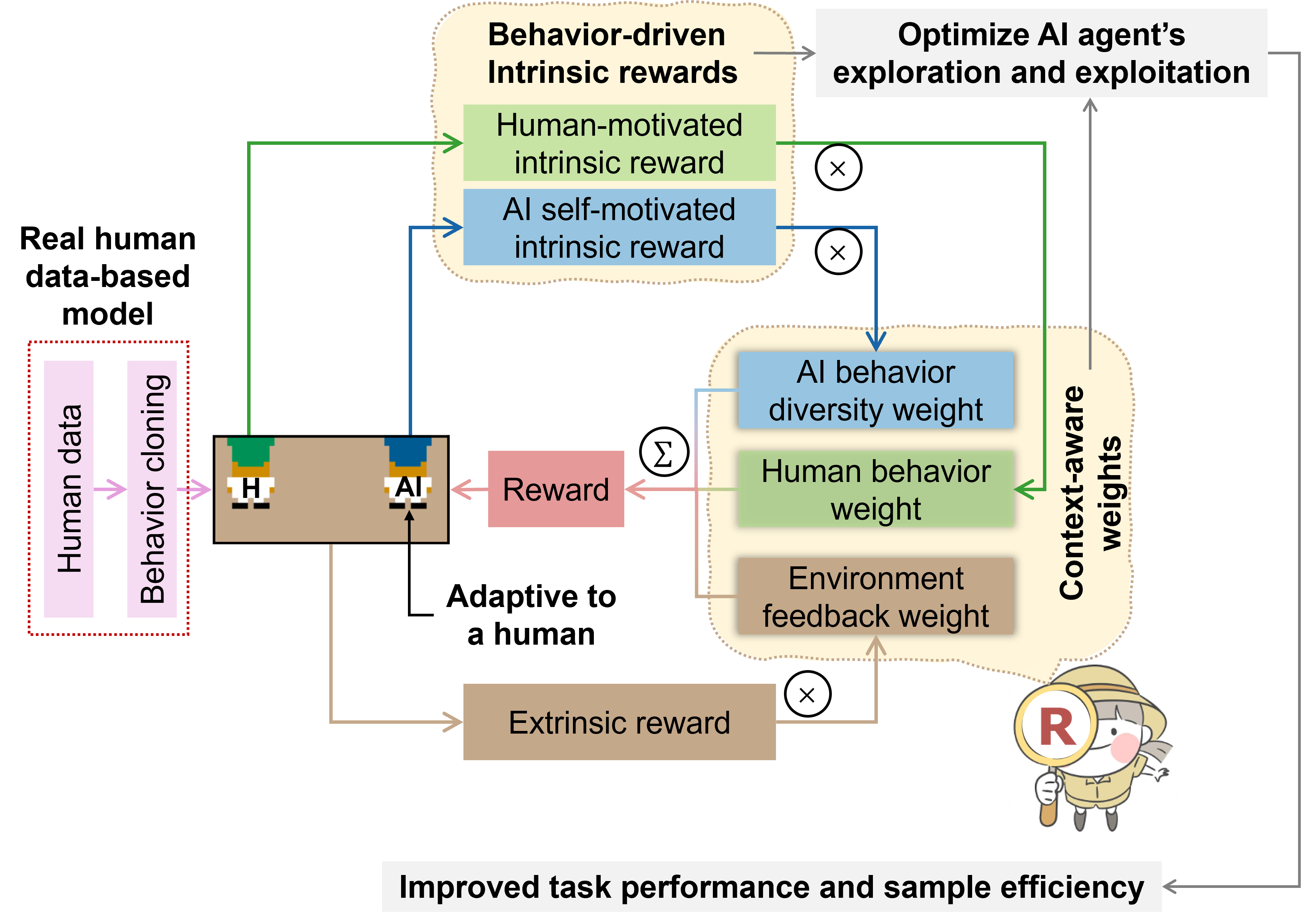}
    \caption{The proposed BCR-DRL for HAIC. The AI agent is trained by our BCR-DRL algorithm, which integrates extrinsic rewards, intrinsic rewards, and context-aware weights. Extrinsic rewards are obtained from the environment, whilst intrinsic rewards are developed based on the behaviors of both human and AI agents. Context-aware weights are dynamically adjusted according to the training context across three domains, such as task performance, AI agent behaviors, and human behaviors.}
    \vspace{15pt}
    \label{fig_BCR-DRL_architecture}
\end{figure}

    \textit{\textbf{ii) Ensuring effective exploitation of the explored state-action pairs.}} 
    Although the above sparse rewarding challenge in HAIC can be mitigated by intrinsic rewards, continuously relying on them throughout training can lead to unstable performance~\cite{NIPS2017_MAAC,Cooperative_IRL,ex_ex_self_supervised}, due to the inherent exploration-exploitation trade-off in DRL~\cite{Sutton_FlowerBook,classical_A3C,ex_ex_MARL}. To address this issue, we can leverage the training context in HAIC to optimize the exploration and exploitation. Specifically, the context in HAIC captures the evolving state of coordination between AI and human participants, such as accumulated sparse rewards, rewards instability conditions, and learning progress. This contextual information enables AI agents to adaptively adjust their learning strategy: prioritizing exploration in early stages when critical state-action pairs are scarce, while gradually transitioning to exploitation in later stages to refine the learned coordination patterns.

In this paper, we design an innovative behavior- and context-aware reward (BCR) for DRL, namely \textit{BCR-DRL}, to address the aforementioned challenges in HAIC. Our BCR (see Fig.~\ref{fig_BCR-DRL_architecture}) extends conventional extrinsic rewards with two key components: (i)~\textit{Dual intrinsic rewards} enhance exploration by encouraging both diverse AI actions and distinct behavior from the human partner, and (ii)~\textit{Context-aware weights} that optimize exploitation by dynamically adjusting reward weights based on the training context.
Specifically, the first component includes two intrinsic rewards: AI self-motivated rewards that promote action diversity, and human-aware rewards that help understand human intentions through counterfactual reasoning. Both rewards utilize logarithmic terms to increase the likelihood of encountering critical but rare state-action pairs. The second component dynamically balances intrinsic and extrinsic rewards based on coordination effectiveness, measured through task completion rates, human behavior patterns, and agent diversity.
Experimental results demonstrate that our BCR-DRL outperforms state-of-the-art algorithms, achieving higher sparse rewards and faster training convergence, which illustrate an innovative path for effective and seamless HAIC.

\section{Related Work} 
\textbf{\textit{Intrinsic rewards.}}
Prior work in multi-agent reinforcement learning has demonstrated the effectiveness of social influence as intrinsic rewards~\cite{ICML2019_Social_MARL_GitHub}. This approach uses counterfactual actions of the ego AI agent to encourage behaviors that significantly influence other AI agents' actions. However, HAIC presents fundamentally different challenges compared to standard MARL scenarios. Unlike MARL where all agents are trainable via RL algorithms, HAIC involves uncontrollable humans whose policies cannot be optimized through RL. This fundamental difference necessitates human behavior modeling approaches (such as behavior cloning) that are incompatible with standard MARL frameworks like centralized training with decentralized execution (CTDE)~\cite{CTDE}.
More recent work has proposed a model-free approach that maximizes AI behavior diversity without explicitly modeling human behavior~\cite{AAAI_overcooked_Tsinghua}. While this approach enables adaptation to general human behaviors, the no free lunch theorems~\cite{NoFreeLunch} suggest that optimal performance in specific scenarios requires exploiting knowledge about the particular behavioral patterns of individuals or groups.

We bridge this gap by proposing a dual intrinsic rewarding scheme that combines the benefits of both approaches while addressing their limitations. Our approach composes: (i) an AI self-motivated intrinsic reward that facilitates comprehensive exploration through behavioral diversity, and (ii) a human-motivated intrinsic reward that utilizes counterfactual human actions obtained from the pre-trained human model, enabling the AI agent to better adapt human intentions. This dual approach enables the AI agent to maintain a comprehensive exploration of its action space while developing actions that effectively complement specific human behavioral patterns, addressing the fundamental constraint that human policies cannot be directly optimized in HAIC settings.

\textbf{\textit{Critical rare state-action pairs.}}\label{section_related_works_rare_pair}
Capturing the critical state-action pairs associated with the sparse rewards is crucial for efficient DRL training in complicated scenarios like HAIC. However, these pairs are usually rarely encountered since they are the results of executing a series of temporally-extended actions~\cite{combine_human_aaai_2019,multiple_feedback_aaai_2023,balance_human_aaai_2020}. Without sufficient exposure of these critical rare state-action pairs, the AI agent struggles to identify effective policies, resulting in slower learning and local optima performance.

To address this challenge, we propose logarithmic-based intrinsic rewards specifically designed for HAIC to encourage targeted exploration. Our key innovation lies in how we leverage the mathematical properties of logarithmic functions: they naturally amplify the impact of low-probability events while compressing high-probability ones. This property is particularly valuable in our context, as it increases the relative importance of rare but critical state-action pairs during training (detailed analysis in Section~\ref{section_BCR_DRL}). This approach enables the agent to explore the state-action space more uniformly, increasing the chances of identifying and utilizing critical state-action pairs.

\textbf{\textit{Adaptive weights.}}
Although incorporating intrinsic rewards can encourage the exploration of the AI agent, it introduces the risk of over-exploration and low data sample utilization efficiency. This issue becomes particularly acute when critical state-action pairs, which were initially rare, but being captured more frequently as training progresses. At this point, continuing to emphasize exploration through intrinsic rewards may hinder rather than help the learning process. Therefore, a desirable DRL reward should be context-aware--capable of recognizing and adjusting when to emphasize the exploration by enlarging the intrinsic rewards and when to emphasize the exploitation by strengthening the extrinsic rewards.

The key challenge here lies in dynamically balancing these competing objectives based on the current training context. This bears a similarity to balancing multiple tasks in deep learning~\cite{Xin_HML}. For example, adaptive weighting has been developed to dynamically adjust task weights based on the instabilities caused by the co-existing tasks in deep learning~\cite{CVPR2019_MultiTask_attention}. Inspired by this concept of adaptive weighting, we propose a novel context-aware weighting mechanism for HAIC. Our approach dynamically adjusts the weights between different rewards by monitoring the training context, including the accumulated sparse rewards and instabilities caused by different rewards, aiming at efficiently achieving the global optimal of sparse rewards. 



\section{Preliminaries}\label{overcooked_env}

\subsection{Foundations for Logarithmic Intrinsic Rewards}\label{section_theory_log}
We provide the theoretical foundations to support the use of logarithmic forms in intrinsic reward design, which will be detailed in eqs.~\eqref{eq_reward_effect_from_AIself} and~\eqref{eq_reward_effect_from_human}. Intuitively, this design focuses on log-likelihood, increasing the sensitivity of the policy to rate state-action pairs.


\begin{theorem}[Entropy-like Logarithmic Intrinsic Rewards for Rare State-action Pair Oriented Policy Updates]
\label{thm:log_reward_unified}
Let $\pi_\theta(a|o)$ be a parameterized stochastic policy over action $a$ given observation $o$. For the counterfactual scenario, let $\tilde{o}$ be a counterfactual observation (e.g., conditioned on human-only action), and $\tilde{a}$ represents the counterfactual action taken with the same policy but under the counterfactual observation. The unified logarithmic intrinsic reward is defined as
\begin{equation}
\mathcal{R}(a, o, \tilde{o}|\pi_\theta, \delta)
\!\!=\!\!
\left| 
\log (\pi_\theta(a|o))
\!-\!
\log (\pi_\theta(\tilde{a}|\tilde{o}))^ \delta
\right|, \delta \in \{0, 1\}, \nonumber
\end{equation}
where $\delta$ determines the specific form of the reward: $\delta = 0$ reduces to $|\log \pi_\theta(a|o)|$, and $\delta = 1$ yields a symmetric difference between log-likelihoods.
As a result, when optimizing the expected intrinsic reward $\mathbb{E}_{\pi_\theta}[\mathcal{R}(a, o, \tilde{o}|\pi_\theta, \delta)]$, the resulting policy gradients naturally amplify updates toward rare state-action pairs, enhancing adaptability in environments with sparse external rewards.
\end{theorem}

\begin{proof}
We prove that the logarithmic intrinsic reward amplifies policy updates toward rare state-action pairs by analyzing the reward scaling properties.

For the case $\delta = 0$, the reward reduces to $\mathcal{R}(a, o|\pi_\theta) = |\log \pi_\theta(a|o)|$. Since $0 < \pi_\theta(a|o) \leq 1$, we have $\log \pi_\theta(a|o) \leq 0$, thus $|\log \pi_\theta(a|o)| = -\log \pi_\theta(a|o)$.

Consider two actions with probabilities $p_1 > p_2 > 0$. The ratio of their logarithmic intrinsic rewards is:
\begin{equation}
\frac{\mathcal{R}(a_2)}{\mathcal{R}(a_1)} = \frac{-\log p_2}{-\log p_1} = \frac{\log(1/p_2)}{\log(1/p_1)}
\end{equation}

Since $p_1 > p_2$, we have $1/p_2 > 1/p_1$, therefore $\frac{\mathcal{R}(a_2)}{\mathcal{R}(a_1)} > 1$, showing that rarer actions receive higher intrinsic rewards.

More importantly, if $p_2 = \epsilon \cdot p_1$ where $0 < \epsilon < 1$, then:
\begin{equation}
\frac{\mathcal{R}(a_2)}{\mathcal{R}(a_1)} = \frac{\log(1/(\epsilon p_1))}{\log(1/p_1)} = 1 + \frac{\log(1/\epsilon)}{\log(1/p_1)}
\end{equation}

As $\epsilon \to 0$ (action becomes rarer), this ratio grows unboundedly, demonstrating exponential amplification for rare actions.

In contrast, traditional entropy-based rewards $\mathcal{H}(a|o) = -\pi_\theta(a|o)\log \pi_\theta(a|o)$ yield:
\begin{equation}
\frac{\mathcal{H}(a_2)}{\mathcal{H}(a_1)} = \frac{p_2 \log(1/p_2)}{p_1 \log(1/p_1)} = \epsilon \cdot \frac{\log(1/(\epsilon p_1))}{\log(1/p_1)}
\end{equation}

As $\epsilon \to 0$, this ratio approaches 0, showing diminishing rewards for rare actions.

The policy gradient contribution is proportional to $\mathcal{R}(a, o|\pi_\theta) \nabla_\theta \log \pi_\theta(a|o)$. For logarithmic rewards, rare actions (small $p$) receive gradient amplification of $O(-\log p)$, while entropy-based methods provide only $O(-p \log p)$, which vanishes as $p \to 0$.

Therefore, the logarithmic form naturally amplifies updates toward rare state-action pairs, enhancing exploration in sparse reward environments. This completes the proof.
\end{proof}

\begin{example} \label{example}
Consider a stochastic policy $\pi_\theta(a|o)$ over $3$ actions in a given observation state $o_1$, with probabilities $\pi_\theta(a_1|o_1)=0.7$, $\pi_\theta(a_2|o_1)=0.2$, and $\pi_\theta(a_3|o_1)=0.1$.
We compare the traditional entropy-based intrinsic reward $\mathcal{H}(a_i|o_i)=-\pi_\theta(a_i|o_i)\!\log(\pi_\theta(a_i|o_i)), \forall i\in\{1,2,3\}$ with our proposed logarithmic intrinsic reward $ \mathcal{R}(a_i,o_i|\pi_\theta) = \left| \log(\pi_\theta(a_i|o_i)) \right|, \forall i\in\{1,2,3\}$ (corresponding to $\delta = 0$ in Theorem~\ref{thm:log_reward_unified}):

\noindent For the traditional entropy-based form, we have\\
$\mathcal{H}(a_1|o_1)= -\pi_\theta(a_1|o_1)\log(\pi_\theta(a_1|o_1)) = 0.250$, \\
$\mathcal{H}(a_2|o_1)= -\pi_\theta(a_2|o_1)\log(\pi_\theta(a_2|o_1)) = 0.322$, \\
$\mathcal{H}(a_1|o_1)= -\pi_\theta(a_3|o_1)\log(\pi_\theta(a_3|o_1)) = 0.230$.

\noindent For our logarithmic form, we have\\
$\mathcal{R}(a_1|o_1)= \left| \log(\pi_\theta(a_1|o_1))\right| = \left|\log(0.7) \right| = 0.357$,  \\
$\mathcal{R}(a_2|o_1)= \left|\log(\pi_\theta(a_2|o_1)) \right| = \left|\log(0.2) \right| = 1.609$,  \\
$\mathcal{R}(a_3|o_1)= \left|\log(\pi_\theta(a_3|o_1))\right|  = \left|\log(0.1) \right| = 2.303$.    
\end{example}
\begin{remark}
In Example~\ref{example}, the traditional entropy-based form, $\mathcal{H}(\cdot)$, provides relatively balanced rewards across all actions, with the highest reward value of $0.322$ for action $a_2$. In contrast, our logarithmic intrinsic reward form, $\mathcal{R}(\cdot)$, particularly amplifies the reward for rare actions, assigning $2.303$ to the least likely action $a_3$, a value approximately $6.5$ times higher than the reward for the most common action $a_1$. Compared with the traditional entropy-based method, this property creates a stronger incentive for policy updates toward exploring rare state-action pairs, which is particularly valuable when critical behaviors might be associated with low-probability actions. By removing the probability weighting term, the logarithmic form establishes a reward landscape that more effectively promotes exploration of the entire state-action space, thereby facilitating the HAIC scenarios that tend to experience sparse rewards.

We also note that the logarithmic form offers computational simplicity by eliminating the probability multiplier term, reducing computational complexity compared to entropy-based methods while maintaining effective exploration capabilities.
\end{remark}

\subsection{Benchmark Human-AI Coordination}\label{section_benchmark_HAIC}
\textbf{\textit{Experimental coordination aspects.}}
We first evaluate our approach using Overcooked~\cite{NeurIPS2019_Overcooked_GitHub,code_github_Overcooked}, a standardized benchmark for HAIC that simulates collaborative cooking tasks requiring coordinated actions between human and AI agents. This environment is particularly suitable for studying HAIC, emphasizing sparse rewards, as successful task completion requires executing precise sequences of coordinated actions of both human and AI agents.
We conduct experiments across three distinct layouts of the overcooked environment. Each of the three layouts highlights different aspects of coordination between human and AI agents. The \textit{Cramped Room} layout focuses on coordination between agents located in a shared space, where collision avoidance between human and AI agents is critical; The \textit{Asymmetric Advantages} layout concentrates on coordination between agents located in distinct areas with varying access to cooking resources, where asymmetric behavior planning is critical; The \textit{Coordination Ring} layout evaluates coordination between agents located in a small room with a central obstacle, where path planning is critical. 

To further investigate the generalization ability of our approach beyond Overcooked, we introduce an \textit{Exploration} environment, where a human and an AI agent must coordinate to explore a shared space by jointly covering all accessible areas within a limited number of timesteps. This setting focuses on coordination in terms of asynchronous and complementary actions, where agents must avoid redundant movements, adapt to each other's partial progress, and efficiently divide the exploration workload. 
Detailed analysis along with the experiments is given in Section~\ref{section_expreiments}.

\textbf{\textit{Human behavior alignment.}}
Current DRL-based AI agents for HAIC fall into two categories: model-free approaches that train AI agents to adapt to general human behaviors without relying on specific human model~\cite{AAAI_overcooked_Tsinghua}, and model-based approaches that enable more personalized adaptation to individual or group-specific behavioral patterns~\cite{NeurIPS2019_Overcooked_GitHub,IJCAI2024_HRL_Overcooked,Harvard_OfflineRL}. 
Our work adopts the model-based approach. The human models utilized in the Overcooked layouts are developed by the authors of~\cite{NeurIPS2019_Overcooked_GitHub}. These real human-cloned models use data collected from a group of humans to imitate the behaviors of human players across three different layouts in the Overcooked environment. The reliability of the selected human models has been validated by the original Overcooked creators~\cite{NeurIPS2019_Overcooked_GitHub}. Subsequent studies~\cite{AAAI_overcooked_Tsinghua,IJCAI2024_HRL_Overcooked} further confirm their ability to reproduce human-level coordination patterns in all three layouts. 
For the generalization study, different from the human-cloned models used in Overcooked, the human model is a deliberately synthetic agent designed to capture more generic and stochastic behaviors, focusing on coordination under temporal asynchrony and behavioral uncertainty. Details of this setup are provided in Section~\ref{experiments_Exploration_env}.

\begin{remark}
    The selection of the model-based approach ensures reliable validation of our BCR-DRL algorithm, particularly its human-motivated intrinsic reward component, which builds upon accurate modeling of human behavior. The three layouts are selected based on their demonstrated reliability in previous work where trained human models achieved performance comparable to average human players~\cite{NeurIPS2019_Overcooked_GitHub}.
    While our BCR-DRL algorithm's performance depends on human model accuracy, it remains robust as long as behavioral patterns are consistent between training and testing phases, regardless of the specific way the human model behaves. This robustness was demonstrated in Section 5.3, where we achieved performance improvements even with a simplified random behavior model that differs significantly from real human behavior patterns.
\end{remark}

\section{The Proposed BCR-DRL Algorithm}\label{section_BCR_DRL}
In this section, we first present the design specifics of our BCR, followed by the training algorithm of BCR-DRL for HAIC.

\subsection{Design Specifics of BCR}\label{section_design_BCR}
Our BCR is defined as 
\begin{equation}
    r_{t} 
    = \kappa_{n}^\mathcal{E} r_t^\mathrm{E} 
    + \kappa_{n}^{\mathrm{A}} r_t^{\mathrm{A}}
    + \kappa_{n}^{\mathrm{H}} r_t^{\mathrm{H}}
    ,
    \label{eq_reward_bcr_drl}
\end{equation}
where $r_t^\mathrm{E}$ is the standard extrinsic reward obtained from the environment in the $t$-th timestep, while $r_t^{\mathrm{A}}$ and $r_t^{\mathrm{H}}$ represent a pair of intrinsic rewards in the $t$-th timestep. The context-aware weights $\kappa_{n}^\mathcal{E}$, $\kappa_{n}^{\mathrm{A}}$, and $\kappa_{n}^{\mathrm{H}}$ modulate the contribution of each reward component, where $n = \lfloor t / T \rfloor$ indicates the training epoch index and $T$ represents the number of timesteps per epoch.

The dual intrinsic rewards encourage the exploration of critical rare state-action pairs that are associated with sparse rewards, capturing distinctive behavioral patterns from both the AI agent and its human coordinator. The context-aware weights adaptively adjust each reward component's significance epoch by epoch, maintaining synchronization with the BCR-DRL policy updates. We note that the use of distinguishable superscripts ($\mathrm{E}$ and $\mathcal{E}$) for the extrinsic reward and its context-aware weight, is intentional to emphasize their distinct design rationale, which will be elaborated in eq.~\eqref{eq_reward_ex_no_discount}.


\setcounter{footnote}{0}
\subsubsection{Extrinsic Reward}\label{section_ex_reward}
The extrinsic reward, $r_t^\mathrm{E}$, comprises two components: the target sparse reward, $r_t^{\mathrm{E}_{\mathrm{S}}}$, and a stage reward\footnote{Stage rewards are defined as the intermediate rewards used to guide the exploration and exploitation of key preliminary actions that lead to the target sparse rewards~\cite{reward_shaping_nips2020,reward_shaping_WED}. We follow the methods of stage rewards designing given in the supplementary of~\cite{NeurIPS2019_Overcooked_GitHub}, thereby saving the space in this paper for discussing its effectiveness.}, $r_t^{\mathrm{E}_{\mathrm{G}}}$. This extrinsic reward is given by
\begin{equation}
r_t^\mathrm{E}\!\!=\!r_t^{\mathrm{E}_{\mathrm{S}}}+r_t^{\mathrm{E}_{\mathrm{G}}}\!\cdot\!f_{\phi}(t)
    \!\!=\!\! 
    \lambda^{\mathrm{E}_{\mathrm{S}}}\!\cdot\!\mathds{1}\left\{f_C(s_t,a_t)\!\!=\!{True}\right\}
    +
    r_t^{\mathrm{E}_{\mathrm{G}}}\!\cdot\!f_{\phi}(t)
    ,
\label{eq_reward_ex}
\end{equation}
where $\lambda^{\mathrm{E}_{\mathrm{S}}}$ is a constant coefficient representing the magnitude of successfully executing critical rare state-action pair associated with sparse reward. The indicator function $\mathds{1}(\cdot)$ evaluates to $1$ when condition $f_C(s_t, a_t) = {True}$ satisfied, indicating the successful execution of action $a_t$ in state~$s_t$. The stage reward term is modulated by a time-dependent fading function $f_{\phi}(t)$. These extrinsic rewards obtained from the environment are shared by human and AI agents.



%
\subsubsection{Intrinsic Reward Design}
Our intrinsic reward design combines two components: an AI self-motivated intrinsic reward and a human-motivated intrinsic reward (see Fig.~\ref{fig_BCR-DRL_architecture}), encouraging comprehensive exploration of critical rate state-action pairs from AI and human behavior patterns.

\textbf{\textit{AI self-motivated intrinsic reward.}}
To encourage the AI agent to adopt a diverse policy that pays more attention to rare state-action pairs, we design the AI self-motivated intrinsic reward following Theorem~\ref{thm:log_reward_unified} as
\begin{equation}
\begin{split}
    r_t^{\mathrm{A}}
        & = \lambda^{\mathrm{A}}
        \cdot
        \mathbb{E}_{\pi}\left[ 
        \left| \log\left(
        \pi(a_t^{\mathrm{A}} \mid o_t^\mathrm{A})        
        \right) \right|
        \right]
            ,
    \label{eq_reward_effect_from_AIself}
\end{split}
\end{equation}
where $\lambda^{\mathrm{A}}$ is a constant coefficient that determines the significance of the self-motivated intrinsic reward, $\mathbb{E}[\cdot]$ denotes the expectation, $\pi(\cdot)$ is the AI agent's policy, and $a_t^{\mathrm{A}}$ and $o_t^{\mathrm{A}}$ represent the AI agent's action and state at the $t$-th timestep, respectively.

\textbf{\textit{Human-motivated intrinsic reward.}}\label{section_reward_intrinsic_human}
We design the human-motivated intrinsic reward based on Theorem~\ref{thm:log_reward_unified}, which is given by
\begin{equation}
    r_t^{\mathrm{H}}
        = 
        \lambda^{\mathrm{H}}
        \cdot
        \mathbb{E}_{\pi}
        \left[
        \left|
        \log
            \left( 
            \frac
            {\pi \left(a_t^{\mathrm{A}} \mid o_t^{\mathrm{A}} \right)}  
            {{\pi} \left(\tilde{a}_t^{\mathrm{A}} \mid \tilde{o}_t^{\mathrm{A}}\left(a_t^{\mathrm{H}}, o_t^{\mathrm{A}}\right)\right)}
            \right) 
        \right|
        \right]
            ,
    \label{eq_reward_effect_from_human}
\end{equation}
where $\lambda^{\mathrm{H}}$ is a constant coefficient, $\tilde{o}_t^{\mathrm{A}}\left(a_t^{\mathrm{H}}, o_t^{\mathrm{A}}\right)$ represents the AI agent's counterfactual observation when only the human takes action $a_t^\mathrm{H}$ at timestep $t$, with the AI agent remaining inactive. The term ${\pi} \left(\tilde{a}_t^{\mathrm{A}} \big| \tilde{o}_t^{\mathrm{A}}\left(a_t^{\mathrm{H}}, o_t^{\mathrm{A}}\right)\right)$ denotes the AI agent's policy in this counterfactual scenario.


Two intrinsic rewards given in eqs.~\eqref{eq_reward_effect_from_AIself} and \eqref{eq_reward_effect_from_human} encourage the AI agent to explore actions that can increase the likelihood of encountering the critical rare state-action pairs associated with sparse rewards. This strategic exploration boosts the frequency of targeted behaviors, helping the AI agent adapt to human behavior effectively in HAIC.

\subsubsection{Context-aware Weights Design}\label{section_context_aware_weights}
As the HAIC training progresses, these critical state-action pairs associated with sparse rewards are encountered more frequently, necessitating a gradual shift from exploration to exploitation compared to early training stages. To facilitate this transition, we design context-aware weights that adaptively balance exploration and exploitation by considering the training context, specifically the accumulated sparse reward values and reward instabilities.

To mitigate potential over-exploration brought about by the intrinsic rewards, intuitively, we limit their influence to the first $N_\mathrm{th}$ training epochs. This truncation is described as
\begin{equation}\label{eq_weights_final}
\begin{split}
    {\kappa}_{n}^\mathcal{E} &= \hat{\kappa}_{n}^\mathcal{E}  \cdot  \mathds{1} \{n < N_\mathrm{th}\} + \mathds{1}\{n \geq N_\mathrm{th}\},
    \\
    {\kappa}_{n}^{\mathrm{A}} &= \hat{\kappa}_{n}^{\mathrm{A}} \cdot \mathds{1} \{n < N_\mathrm{th}\},
    \\
    {\kappa}_{n}^{\mathrm{H}} &= \hat{\kappa}_{n}^{\mathrm{H}} \cdot \mathds{1} \{n < N_\mathrm{th}\},
\end{split}
\end{equation}
where $\hat{\kappa}_{n}^\mathcal{E}$, $\hat{\kappa}_{n}^\mathrm{A}$, and $\hat{\kappa}_{n}^\mathrm{H}$ represent the influencing degrees regarding the HAIC across three domains--task performance, AI agent behavior, and human behavior, respectively. These scores reflect the degrees of influence from sparse rewards, AI self-motivated intrinsic rewards, and human-motivated intrinsic rewards, respectively, during the training phase when all three components are active.

\begin{algorithm}[t]
\caption{The Proposed BCR-DRL Algorithm}\label{alg_BCR_DRL}
\begingroup
\begin{algorithmic}[1]
\Require Episode number $E$ per epoch, maximum steps per episode $K$, number of epochs for NN updating $N$, discount factor $\gamma$, smoothing factor $\alpha$ of generalized advantage estimator, clip factor $\omega$. $T=EK$.
\Ensure Well-trained actor network $\pi(\cdot|\cdot;\theta^*)$.
\State Initialize actor network $\pi(\cdot|\cdot)$ and critic network $V(\cdot)$ with random parameter $\theta$ and $\varphi$, respectively.
\For{$n = 0$ to $N$} 
    \For{episode = $1$ to $E$} 
        \State Randomly initialize the HAIC scenario.
        \For{$t = nT+1$ to $nT+K$} 
            \State Collect $<o_t^{\mathrm{A}},\tilde{o}_t^{\mathrm{A}}\left(a_t^{\mathrm{H}}, o_t^{\mathrm{A}}\right), a_t^{\mathrm{A}},r_t>$ 
            \State Calculate the rewards in eqs.~\eqref{eq_reward_ex}, \eqref{eq_reward_effect_from_AIself}, and \eqref{eq_reward_effect_from_human}.\label{alg:experiences}
        \EndFor
    \EndFor
    \State Update context-aware weights in eq.~\eqref{eq_weights_final}.
        \State Compute the advantage function $A_t$.
        \State Compute the reward-to-go based on reward in eq.~\eqref{eq_reward_bcr_drl}.
    \State Update parameters of actor and critic NNs, $\varphi$ and $\theta$.
\EndFor
\end{algorithmic}
\endgroup
\end{algorithm}

To measure these three influencing degrees in eq.~\eqref{eq_weights_final}, we propose to assign greater values to rewards exhibiting higher instability. For example, extrinsic reward is assigned a larger influencing degree if they show greater instabilities compared to intrinsic rewards, and vice versa. To ensure policy robustness, these influencing degrees are updated epoch-by-epoch in synchronization with the BCR-DRL updates, and are calculated by
\begin{equation}\label{eq_weights}
    \hat{\kappa}_{n}^\mathcal{E}, 
    \hat{\kappa}_{n}^{\mathrm{A}}, 
    \hat{\kappa}_{n}^{\mathrm{H}}
    = 
    \lambda^\mathrm{R} 
    \cdot
    \texttt{softmax} 
    \left(
    \frac
    {\Bar{R}_{n-1}^\mathcal{E}}
    {\Bar{R}_n^\mathcal{E}}
    , 
    \frac
    {\Bar{R}_{n-1}^{\mathrm{A}}}
    {\Bar{R}_n^{\mathrm{A}}}
    , 
    \frac
    {\Bar{R}_{n-1}^{\mathrm{H}}}
    {\Bar{R}_n^{\mathrm{H}}}
    \right),
\end{equation}
where $\lambda^\mathrm{R}$ is a constant coefficient. The terms $\Bar{R}_n^\mathcal{E}$, $\Bar{R}_n^\mathrm{A}$ and $\Bar{R}_n^\mathrm{H}$ are: the average summation of sparse and stage rewards; the average AI self-motivated intrinsic rewards; and the average human-motivated intrinsic rewards, respectively. They are calculated by
\begin{equation}
\begin{split}
    \Bar{R}_n^\mathcal{E}       
    = \frac{1}{T}
    \hspace{-6pt}
    \sum_{t=nT}^{(n+1)T} 
    \hspace{-7pt}
    r_t^\mathcal{E}
    ,
    ~~~
    \Bar{R}_n^{\mathrm{A}} 
    = 
    \frac{1}{T}
    \hspace{-6pt}
    \sum_{t=nT}^{(n+1)T}
    \hspace{-7pt}
    r_t^{\mathrm{A}}
    ,
    ~~~
    \Bar{R}_n^{\mathrm{H}} 
    &= 
    \frac{1}{T}
    \hspace{-6pt}
    \sum_{t=nT}^{(n+1)T}
    \hspace{-7pt}
    r_t^{\mathrm{H}}
    ,
\end{split}
\end{equation}
where
\begin{equation}
    r_t^\mathcal{E} = r_t^{\mathrm{E}_\mathrm{S}} + r_t^{\mathrm{E}_\mathrm{G}},
\label{eq_reward_ex_no_discount}
\end{equation}
is the summation of sparse and stage rewards, distinct from the extrinsic reward defined in eq.~\eqref{eq_reward_ex}. The design intuition is environmental dynamics, represented by ${\Bar{R}_{n-1}^\mathcal{E}}/ {\Bar{R}_n^\mathcal{E}}$ in eq.~\eqref{eq_weights}, cannot be adequately represented by only sparse reward or extrinsic reward solely. 


\subsection{Training Algorithm}
Our BCR-DRL agent is adapted from proximal policy optimization (PPO)~\cite{ppo_classic}. Its architecture consists of actor and critic neural networks with parameters $\theta$ and $\varphi$, respectively. The actor generates actions $a_t^{\mathrm{A}}$ using a stochastic policy $\pi(a_t^{\mathrm{A}}|o_t^{\mathrm{A}};\theta)$, which represents the probability density of $a_t^{\mathrm{A}}$ given the current state $o_t^{\mathrm{A}}$. The critic estimates the state-value function $V(o_t^{\mathrm{A}};\varphi)$ based on the actor's policy. 
Please refer to Algorithm~\ref{alg_BCR_DRL} for the step-by-step training process of BCR-DRL for HAIC.

\begin{figure*}[!t]
    \centering
    \begin{subfigure}[t]{5.9cm}
        \includegraphics[width=5.9cm]{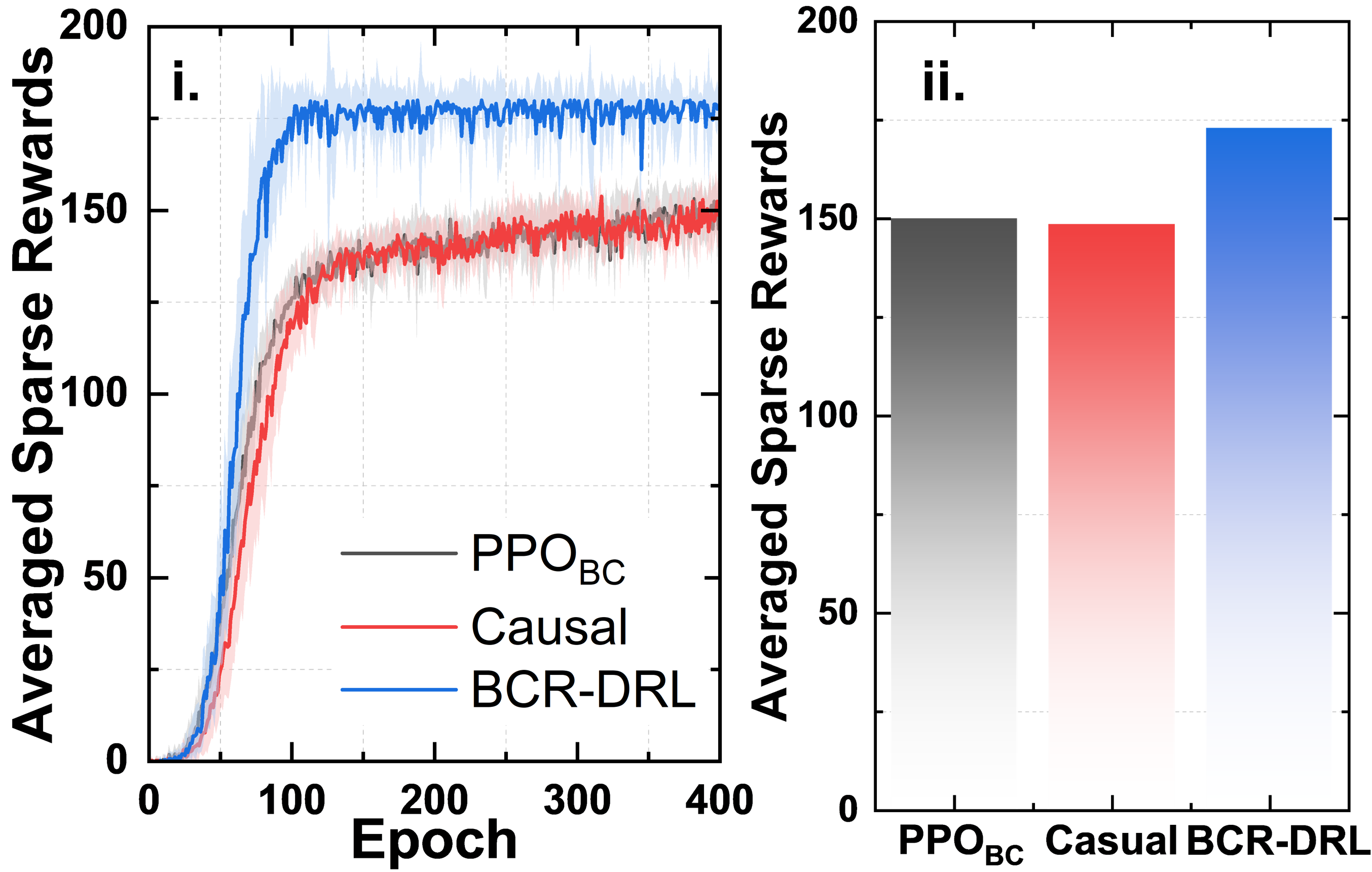}
        \subcaption{\textit{Cramped Room}.}
        \label{fig_cr_sim_results1}
    \end{subfigure}
    \begin{subfigure}[t]{5.9cm}
        \includegraphics[width=5.9cm]{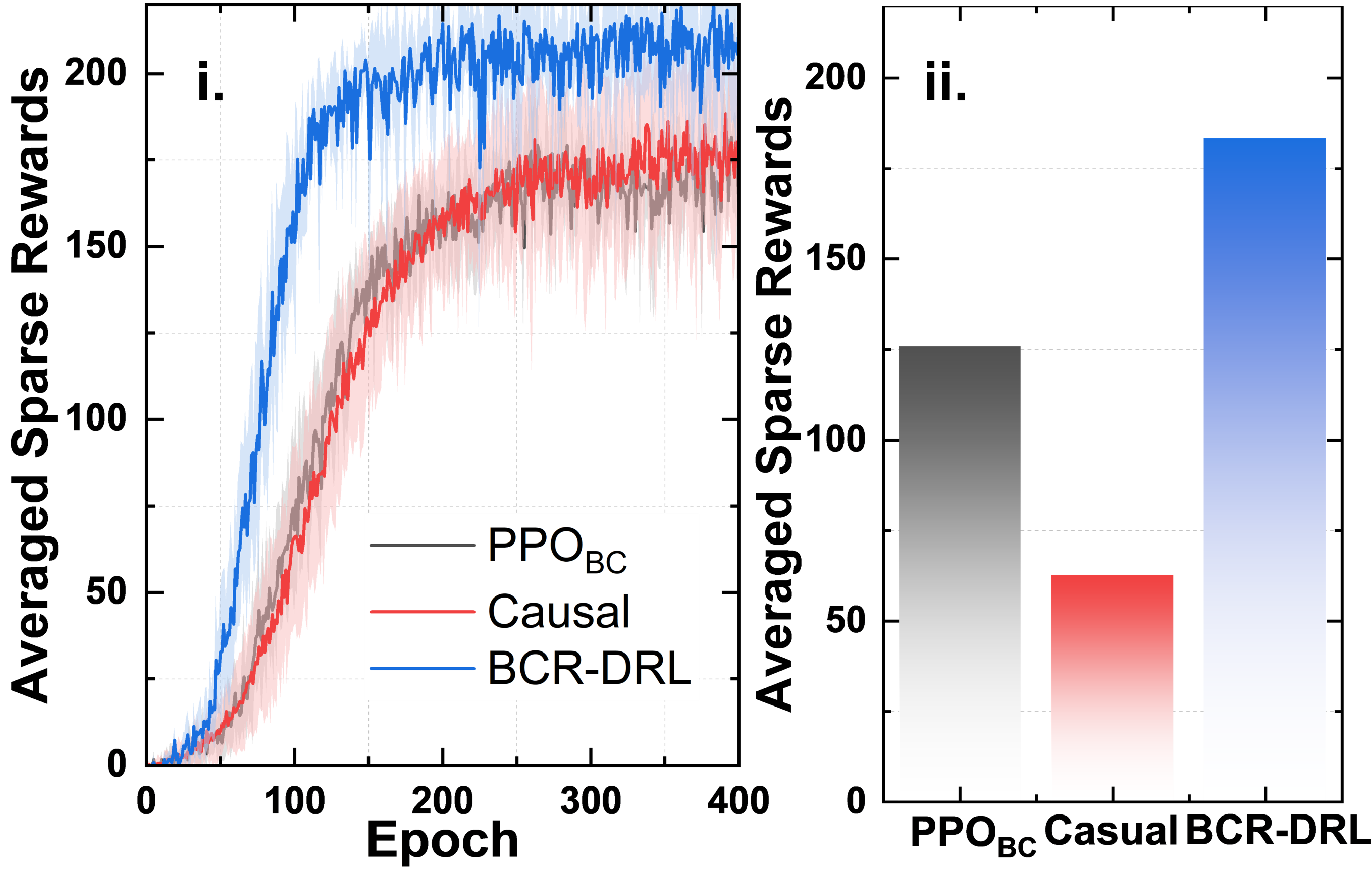}
        \subcaption{\textit{Asymmetric Advantages}.}
        \label{fig_aa_sim_results1}
    \end{subfigure}
    \begin{subfigure}[t]{5.9cm}
        \includegraphics[width=5.9cm]{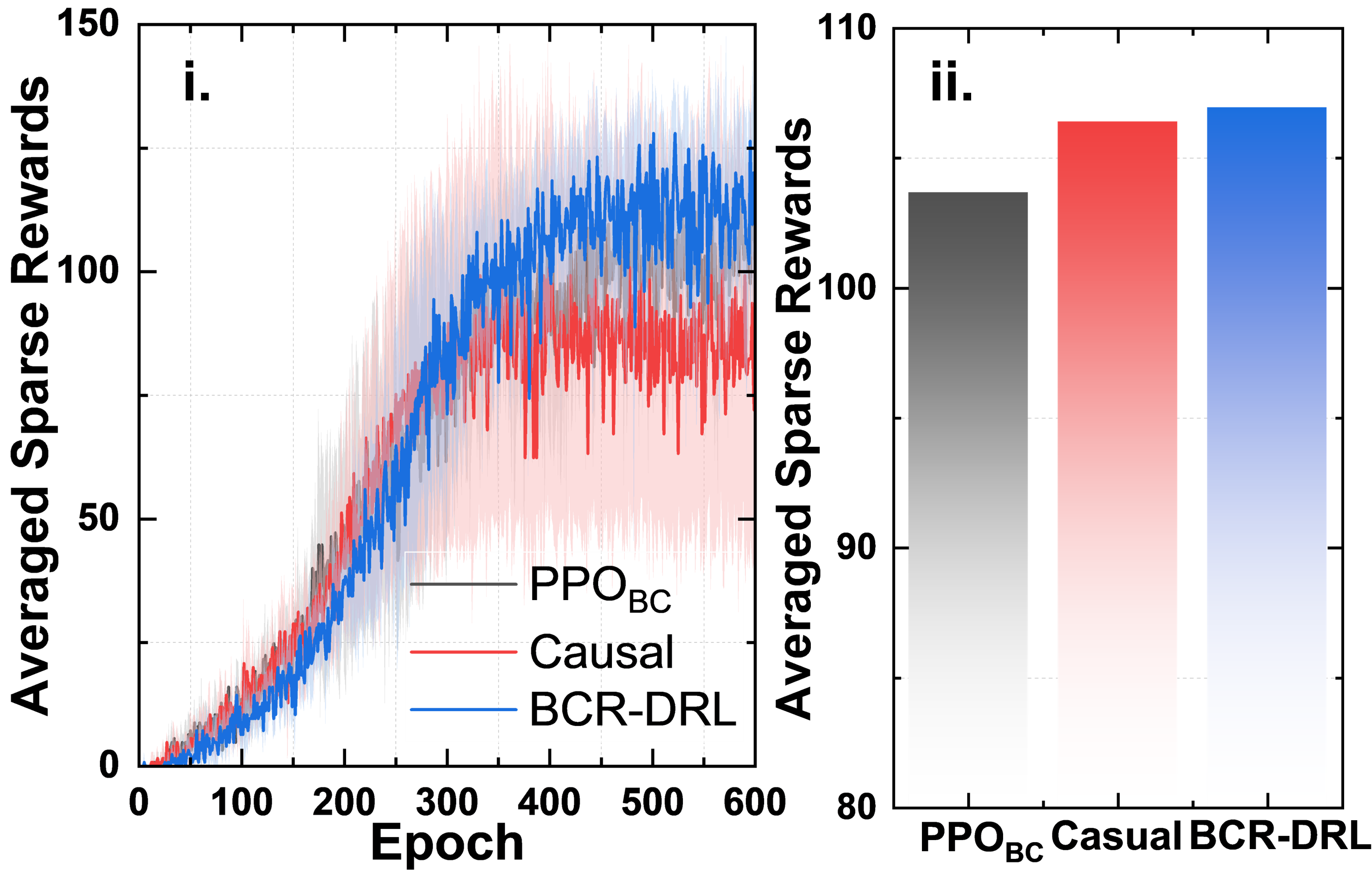}
        \subcaption{\textit{Coordination Ring}.}
        \label{fig_co_sim_results1}
    \end{subfigure}
    \vspace{10pt}
    \caption{Average sparse rewards in different layouts of the Overcooked environment, obtained by training in~(i) and testing in~(ii).}
    \label{fig_all_sparse_reward}
    \vspace{10pt}
\end{figure*}


\section{Experiments}\label{section_expreiments}
Following the original Overcooked environment setup~\cite{NeurIPS2019_Overcooked_GitHub}, our experimental framework consists of one human model and one RL-based AI agent, making our approach effectively a single-agent reinforcement learning problem where we train a single-agent PPO to coordinate with the human model. We note that while the training process uses a combination of multiple reward components (extrinsic and intrinsic rewards), we report only the sparse environmental rewards during both training and evaluation phases, as they serve as the primary metric for task completion performance. The detailed training curves for individual reward components are provided in Fig. 4.

We evaluate \textbf{BCR-DRL} algorithm against two benchmark algorithms: 1) \textbf{PPO$_\mathbf{BC}$ benchmark:} This benchmark uses only the extrinsic reward, without incorporating intrinsic rewards or context-aware weights adjustments~\cite{NeurIPS2019_Overcooked_GitHub}. 2) \textbf{Causal benchmark:} This benchmark includes causal influence rewards in addition to the extrinsic rewards used in PPO$_\mathbf{BC}$ benchmark. We include this benchmark to analyze the effectiveness of intrinsic rewards that solely focus on social influence as intrinsic motivation proposed in~\cite{ICML2019_Social_MARL_GitHub}, without incorporating context-aware weights.
All algorithms utilize the standard PPO framework~\cite{ppo_classic,code_keras_ppo}, and are evaluated with corresponding human models tailored for individual scenarios. 
For fair comparisons, all algorithms share the same neural network architecture and hyperparameter settings in the same scenario.

\begin{figure*}[!t]
    \centering
    \begin{subfigure}[t]{5.9cm}
        \includegraphics[width=5.3cm]{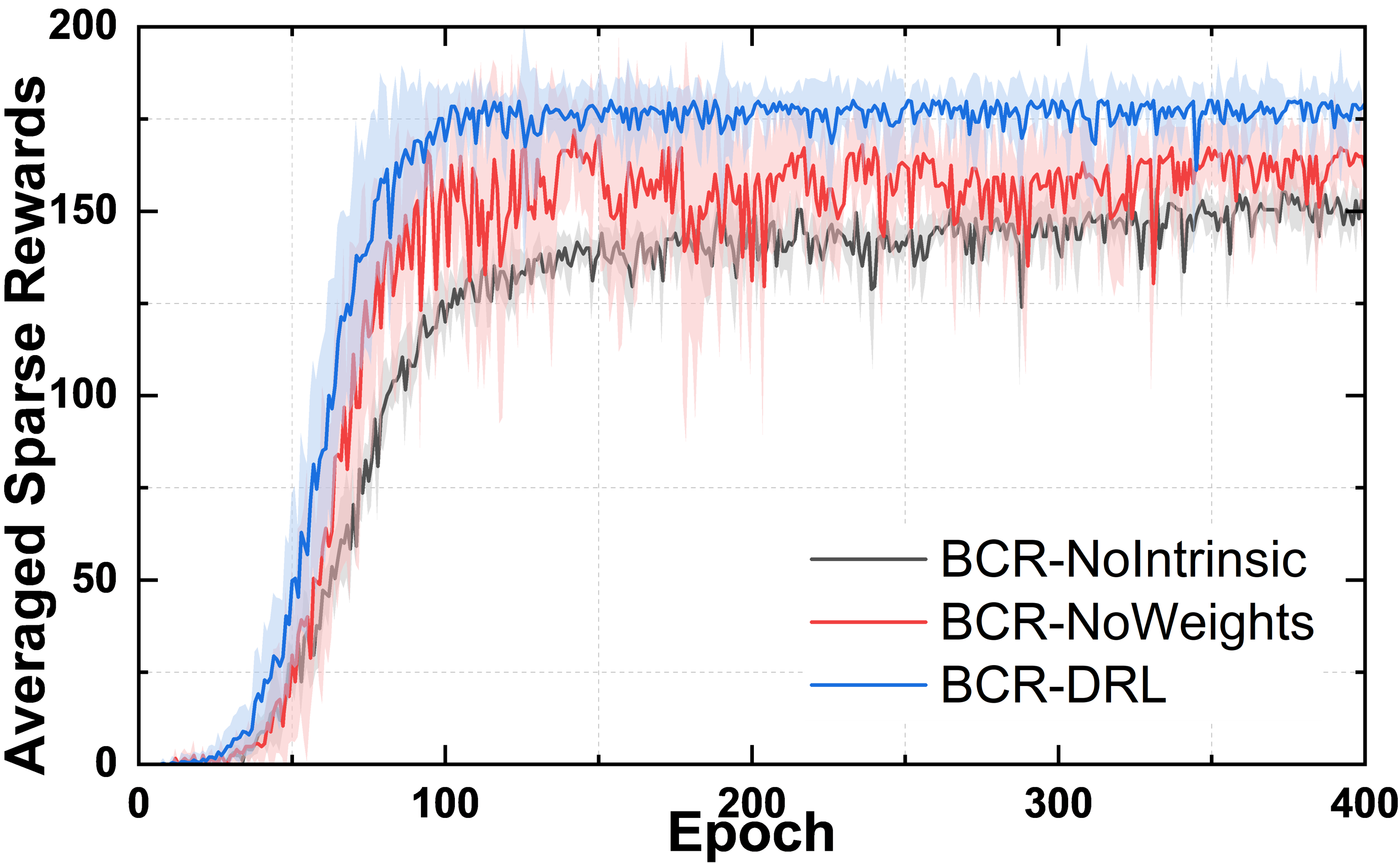}
        \subcaption{\textit{Cramped Room}.}
        \label{fig_ablation_CR}
    \end{subfigure}
    \begin{subfigure}[t]{5.9cm}
        \includegraphics[width=5.3cm]{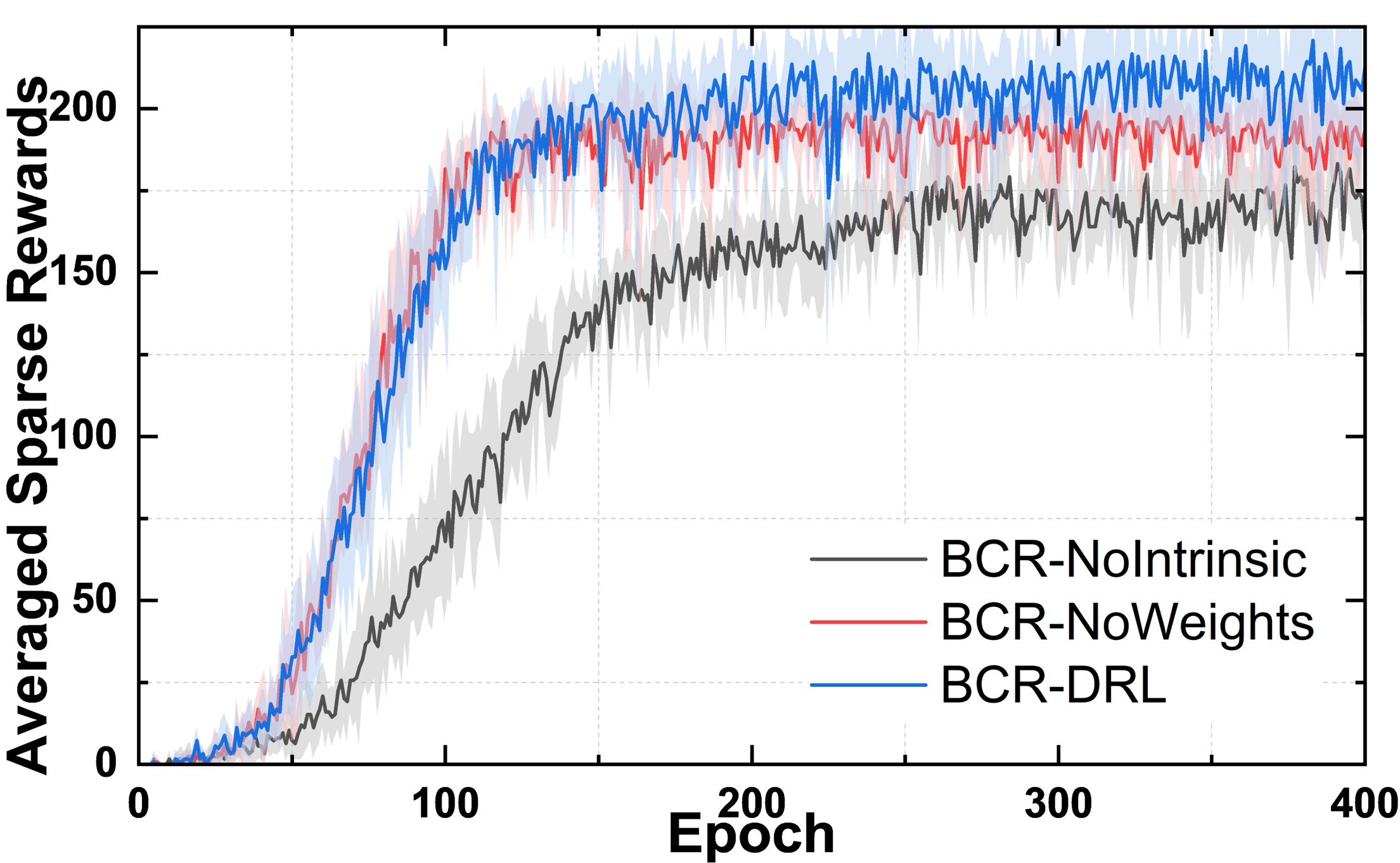}
        \subcaption{\textit{Asymmetric Advantages}.}
        \label{fig_ablation_AA}
    \end{subfigure}
    \begin{subfigure}[t]{5.9cm}
        \includegraphics[width=5.5cm]{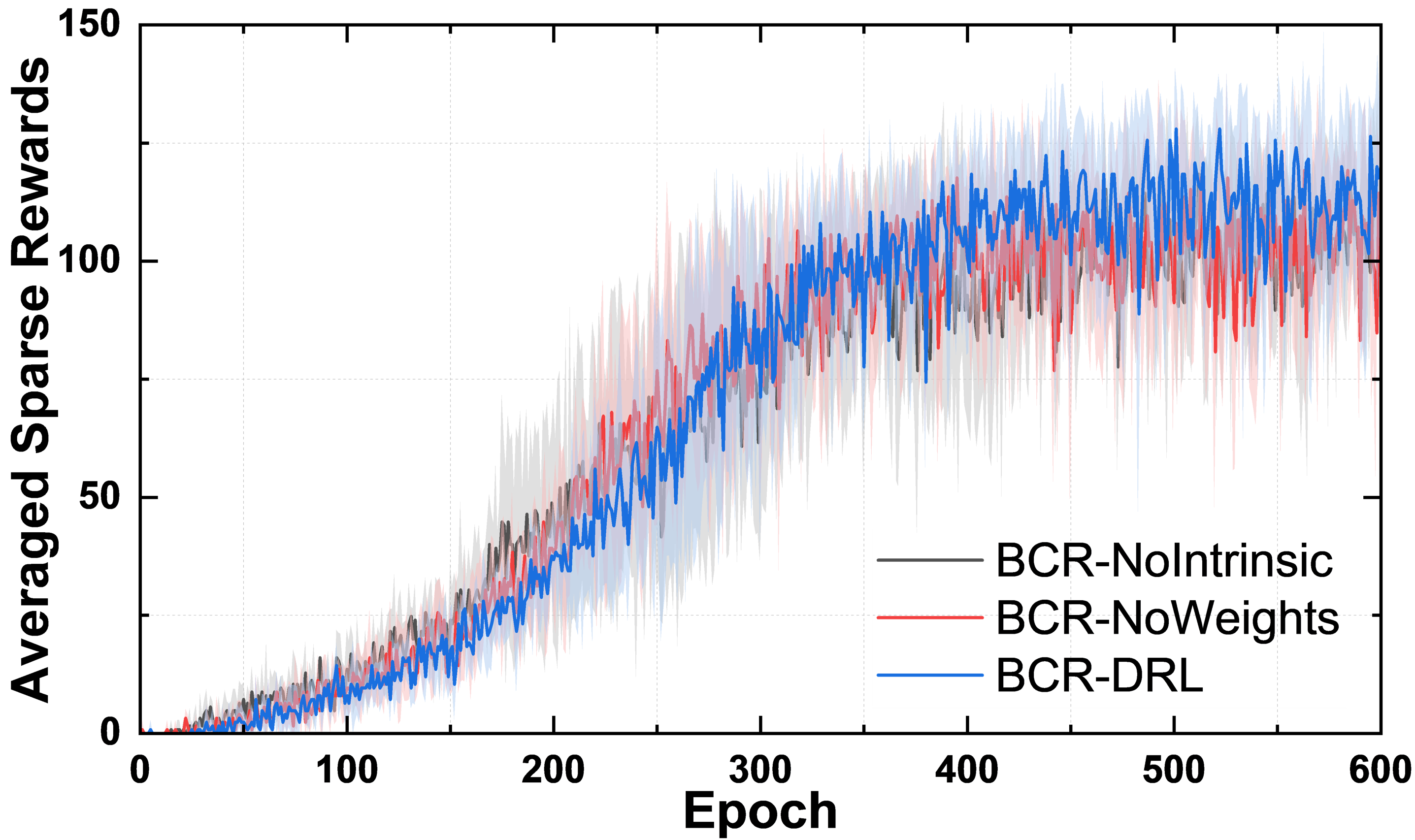}
        \subcaption{\textit{Coordination Ring}.}
        \label{fig_ablation_CO}
    \end{subfigure}
    \vspace{10pt}
    \caption{Ablation studies comparing our BCR-DRL full model against two variants, \texttt{BCR-NoIntrinsic} and \texttt{BCR-NoWeights}.
    The comparison between BCR-DRL and BCR-NoIntrinsic demonstrates the effectiveness of intrinsic rewards in improving performance. The comparison between BCR-DRL and BCR-NoWeights shows the stabilizing effect of our context-aware weighting mechanism, which mitigates potential instabilities from intrinsic rewards by preventing them from dominating later-stage learning.
    }
    \label{fig_ablation_study}
    \vspace{10pt}
\end{figure*}


\begin{figure*}[!ht]
    \centering
    \begin{subfigure}[t]{5.8cm}
        \includegraphics[width=5.8cm]{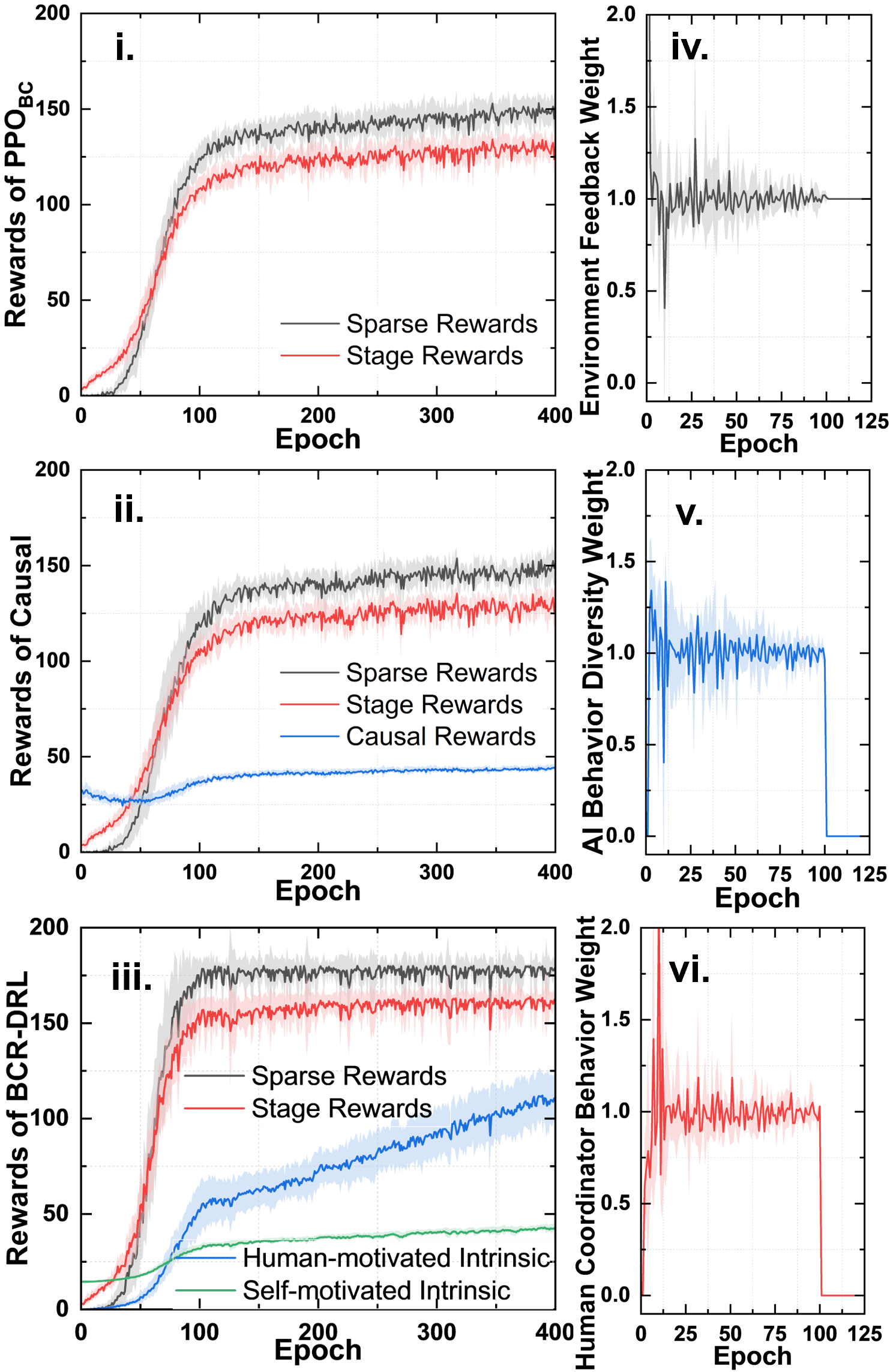}
        \vspace{-12pt}
        \subcaption{\textit{Cramped Room}.}
        \label{fig_cr_sim_results2}
    \end{subfigure}
    \begin{subfigure}[t]{5.9cm}
        \includegraphics[width=5.9cm]{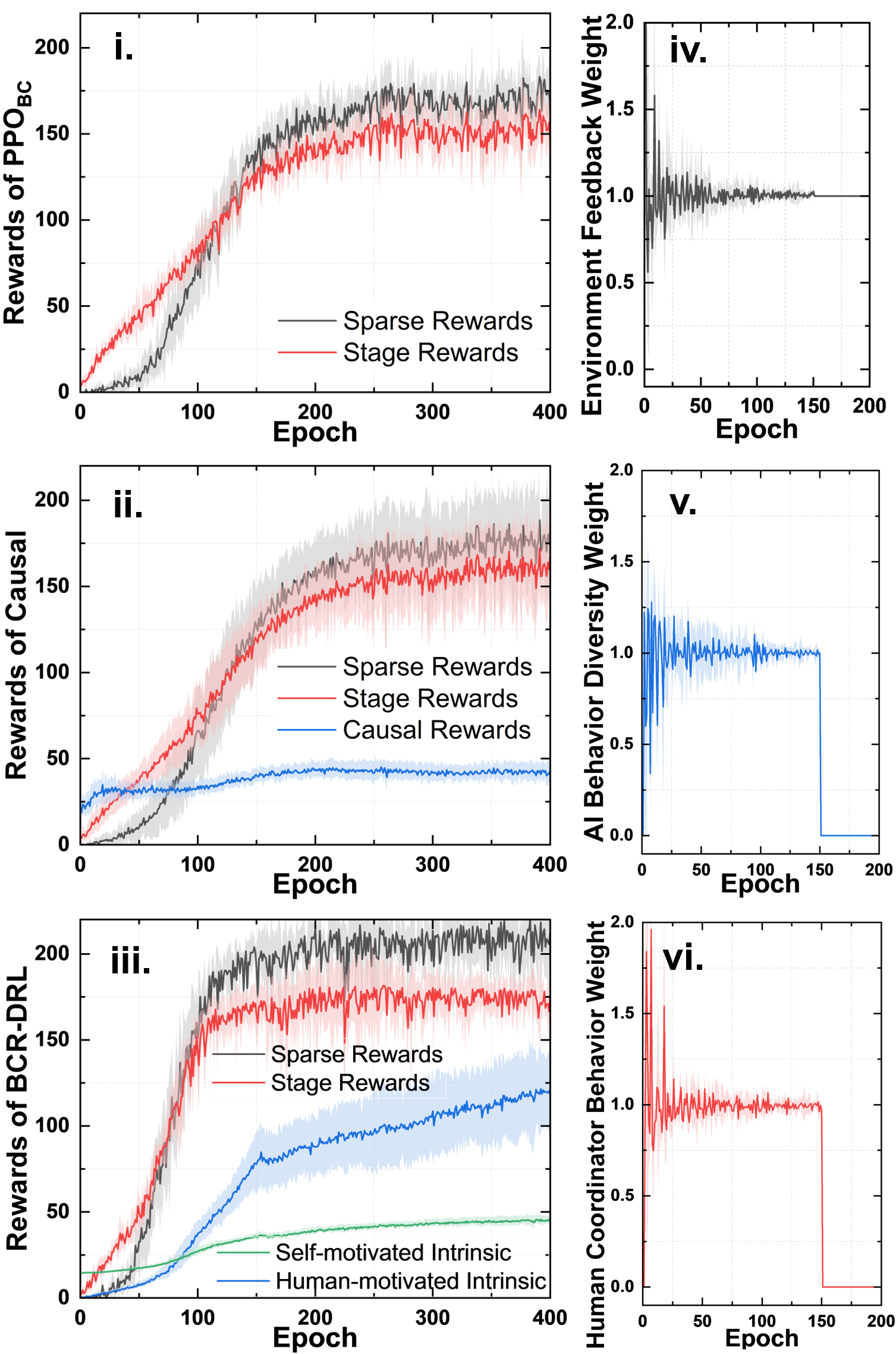}
        \vspace{-12pt}
        \subcaption{\textit{Asymmetric Advantages}.}
        \label{fig_aa_sim_results2}
    \end{subfigure}
    \begin{subfigure}[t]{5.9cm}
        \includegraphics[width=5.9cm]{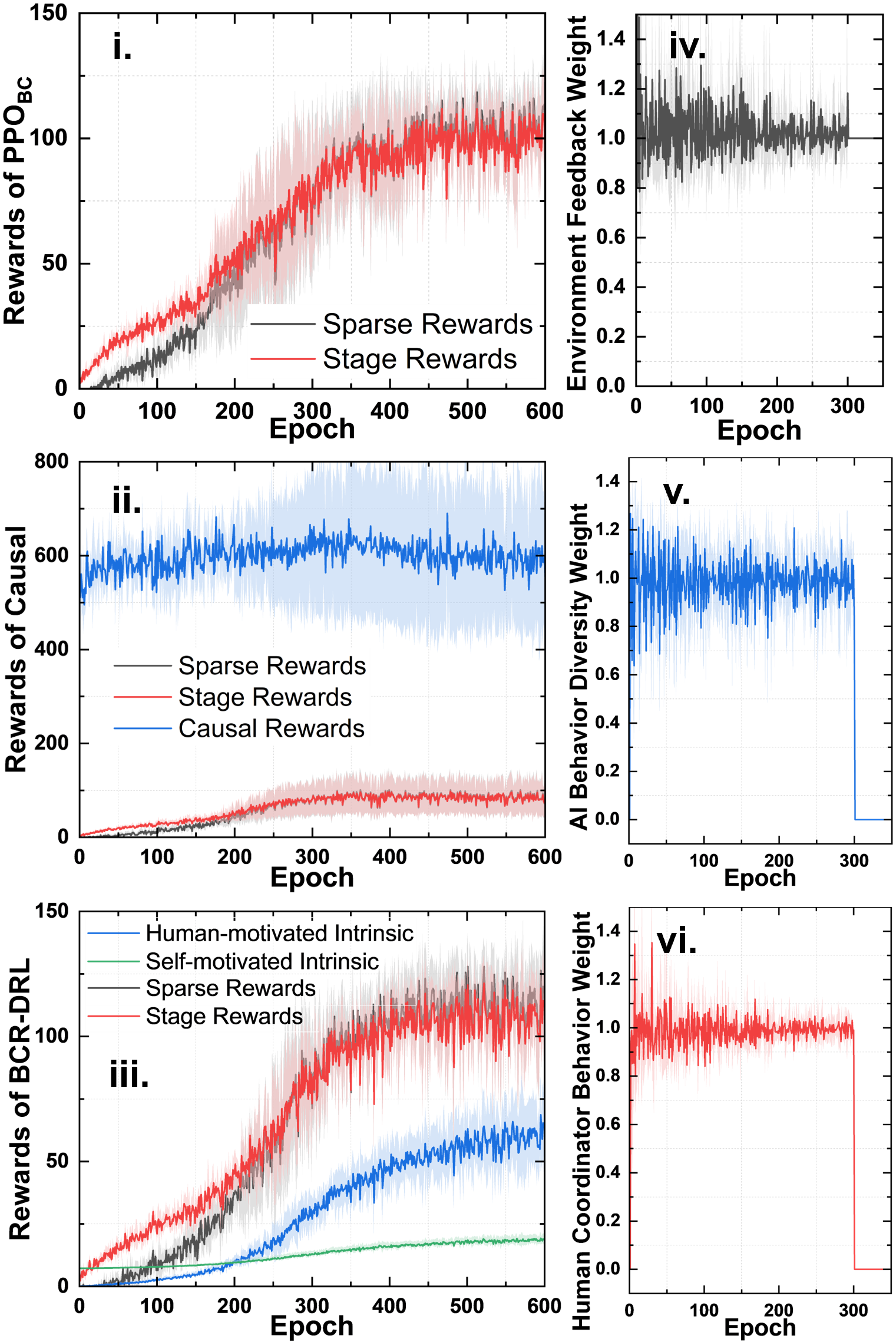}
        \vspace{-12pt}
        \subcaption{\textit{Coordination Ring}.}
        \label{fig_co_sim_results2}
    \end{subfigure}
    \vspace{10pt}
    \caption{Performance analysis of individual components in BCR. Impacts of the proposed methods in the \textit{Cramped Room}, \textit{Asymmetric Advantages}, and \textit{Coordination Ring} layouts, respectively. The intrinsic rewards are given in~(i)-(iii) of each figure, whilst the context-aware weights are given in~(iv)-(vi). Specifically, the Rewards correspond to three algorithms: PPO$_\mathbf{\mathrm{BC}}$~(i), Causal~(ii), and BCR-DRL~(iii); whilst the Weights compose the three context-aware weights of the BCR-DRL given in eq.~\eqref{eq_weights_final}, representing the environment feedback weight ${\kappa}_{n}^\mathcal{E}$~(iv), the AI behavior diversity weight ${\kappa}_{n}^\mathrm{A}$~(v), and the human coordinator behavior weight ${\kappa}_{n}^\mathrm{H}$~(vi), respectively.}
    \vspace{5pt}
    \label{fig_effectivenness_analysis}
\end{figure*}

\subsection{Performance on Overcooked Environment}
The overcooked environment~\cite{NeurIPS2019_Overcooked_GitHub} presents cooking tasks to be accomplished by coordinated actions of human and AI agents within limited timesteps. Specifically, a human agent and an AI agent work together to prepare as many onion soups as possible within a limited number of timesteps. A shared sparse reward ($r_t^{\mathrm{E}}$ given in eq.~\eqref{eq_reward_ex}) of $20$ points will be granted for both agents upon each successfully served onion soup, which requires completing a sequence of actions, including picking up onions from specified locations, placing three onions into a pot, picking up the soup with a plate after cooking it for $20$ timesteps, and serving the soup to a designated area.

We evaluate our BCR-DRL approach across three distinct Overcooked layouts, each highlighting different aspects of HAIC. Our experiments employ model-based HAIC using human models trained through behavior cloning algorithms. Each layout utilizes different sets of real human interaction data sourced from~\cite{code_github_Overcooked,NeurIPS2019_Overcooked_GitHub}, with separate human models used during training and testing phases to investigate robustness. Our implementation leverages the Gym-compatible Overcooked environment with Tensorflow~\cite{code_github_Overcooked}.

To ensure statistical reliability and demonstrate the stability of our approach, including the context-aware weighting mechanism, all experimental results are based on comprehensive statistical analysis across multiple independent runs. Specifically, for each experimental layout, we provide statistical experimental training and testing results, based on $5$ independent training runs and $2000$ independent testing episodes, respectively. This rigorous experimental protocol ensures that our reported performance improvements and the stability of context-aware weights are statistically significant and reproducible across different training conditions.




\subsubsection{Sparse Reward Performance Comparison}\label{section_sim_sparse}
Fig.~\ref{fig_all_sparse_reward} shows the comparative performance of sparse rewards between our BCR-DRL and the benchmark algorithms.

\textbf{\textit{Cramped Room.}}
As shown in Fig.~\ref{fig_cr_sim_results1}, our BCR-DRL consistently outperforms benchmark methods, and achieves around $20\%$ higher average sparse rewards in both training and testing phases. 
While the PPO$_\mathrm{BC}$ benchmark provides a solid foundation through its focus on the extrinsic rewards, and the Causal benchmark offers valuable insights through counterfactual actions of the ego AI agent, our BCR-DRL builds upon these approaches by designing a more effective and efficient solution for HAIC in this shared-space HAIC scenario.

\textbf{\textit{Asymmetric Advantages.}}
Fig.~\ref{fig_aa_sim_results1} shows that our BCR-DRL achieves consistently higher sparse rewards compared to both benchmarks, during both training and testing phases. The algorithm's effectiveness in this layout is particularly evident in its convergence speed. 
Using a 10-epoch moving average and a $90\%$ plateau threshold, the BCR-DRL converges at approximately 130 epochs, whereas the benchmarks converge at around 210 epochs, indicating an approximately $38\%$ improvement on the sample efficiency.
Interestingly, for the testing results, Fig.~\ref{fig_aa_sim_results1} reveals a significant performance drop in the Causal benchmark compared with the other two algorithms. This performance gap can be attributed to the difference between training and testing human models, which introduces distributional shifts that the Causal benchmark fails to handle effectively. This observation underscores the robustness side of our context-aware weighting mechanism in HAIC, particularly when intrinsic rewards are integrated to enhance DRL exploration in Fig.~\ref{fig_aa_sim_results1}.

\textbf{\textit{Coordination Ring.}}
As shown in Fig.~\ref{fig_co_sim_results1}, while BCR-DRL still outperforms the benchmarks, the performance improvement is less pronounced compared to the previous two layouts. This can be attributed to the confined nature of the \textit{Coordination Ring}, which naturally constrains the state space and enables the PPO$_\mathrm{BC}$ benchmark to achieve relatively comprehensive exploration even without sophisticated reward mechanisms.
An interesting observation is that the Causal benchmark exhibits higher variance compared to both BCR-DRL and PPO$_\mathrm{BC}$, suggesting that purely partner behavior-motivated intrinsic rewards may be less stable in environments with a confined space that is easy to explore comprehensively.
These observations in the \textit{Coordination Ring} layout suggest that while traditional PPO$_\mathrm{BC}$ demonstrates competitive performance in confined environments that are relatively straightforward to explore comprehensively, BCR-DRL offers a more robust and generalizable approach, capable of adapting to the diverse challenges inherent in HAIC scenarios.



\subsubsection{Ablation Studies}\label{section_ablation_study}
Our ablation study examines BCR's key components across three Overcooked layouts by comparing our full model against: \texttt{BCR-NoIntrinsic} (without intrinsic rewards) and \texttt{BCR-NoWeights} (with intrinsic rewards but without context-aware weights).
Fig.~\ref{fig_ablation_study} demonstrates that component impact varies according to the coordination requirements of each layout. In the \textit{Cramped Room} (Fig.~\ref{fig_ablation_CR}), \texttt{BCR-NoWeights} initially matches the full model but becomes unstable after 100 epochs, confirming our design intuition from Section~\ref{section_context_aware_weights} that unweighted intrinsic rewards become detrimental during exploitation phases. For \textit{Asymmetric Advantages} (Fig.~\ref{fig_ablation_AA}), \texttt{BCR-NoWeights} performs adequately but underperforms during later training stages when balanced reward emphasis becomes critical. In the \textit{Coordination Ring} (Fig.~\ref{fig_ablation_CO}), all variants achieve similar performance, indicating that confined spaces require less sophisticated exploration.
Note that \texttt{BCR-NoIntrinsic} is equivalent to the PPO$_\mathrm{BC}$ benchmark (see Section~\ref{section_sim_sparse} for detailed analysis). These ablation studies validate that dual intrinsic rewards enhance critical exploration in the early training stages, while context-aware weights improve stability during exploration-exploitation transitions, with benefits most significant in scenarios requiring comprehensive exploration and extensive agent coordination.

\subsubsection{Effectiveness Analysis on BCR Design}\label{section_effectivness_analysis}
As shown in Fig.~\ref{fig_effectivenness_analysis}, we analyze the effectiveness of BCR's key components through their training curves. We recall that the definitions of all components are given in Section~\ref{section_design_BCR}.

\textbf{\textit{Cramped Room.}}
In Fig.~\ref{fig_cr_sim_results2}(ii)-(iii), BCR-DRL consistently outperforms the Causal benchmark in intrinsic rewards, demonstrating enhanced exploration of human behavior patterns. 
The component weight evolution in Figs.~\ref{fig_cr_sim_results2}(iv)-(vi) reveals a systematic adaptation process. During early training, environment feedback (iv) and behavior diversity weights (v) dominate as the AI effectively leverages direct environmental information in this shared space. As training progresses, human coordinator behavior weights (vi) gradually increase, reflecting the AI's focus shifting with the increase in human behavioral data. This adaptive weighting illustrates how the algorithm efficiently transitions from environmental exploration to HAIC in a shared space.

\textbf{\textit{Asymmetric Advantages.}}
Fig.~\ref{fig_aa_sim_results2}(iii) shows human-motivated intrinsic rewards with a distinctive inflection point, reflecting the context-aware weighting mechanism's ability to modulate reward signals during the exploration-exploitation transition.
Unlike in \textit{Cramped Room}, the human coordinator behavior weights (vi) reach significant values earlier in the training process. This stems from the environment's spatial separation and asymmetric resource distribution, which make human behaviors informative for policy learning. The adaptive process reflects how BCR-DRL improves AI's adaptation in environments with spatial differences and asymmetric access to critical resources.

\textbf{\textit{Coordination Ring.}}
Fig.~\ref{fig_co_sim_results2} shows distinct patterns, with the Causal benchmark exhibiting significantly larger intrinsic rewards with high variance, explaining its unstable performance in Fig.~\ref{fig_co_sim_results1}. BCR-DRL's human-motivated intrinsic reward maintains relatively consistent levels throughout training, suggesting this confined space is naturally easier to explore, limiting BCR's exploration advantages. 
Context-aware weights demonstrate greater stability compared to other layouts, with slightly lower environmental feedback weights and higher AI behavior diversity weights initially, indicating BCR-DRL prioritizes diverse movement strategies over environment exploration in early training.

\subsection{Hyperparameter Selection and Sensitivity Methodology}\label{section_hyperparameter_analysis}
We present a systematic analysis of BCR-DRL's critical hyperparameters: intrinsic reward truncation epoch ($N_{\mathrm{th}}$), intrinsic reward coefficients ($\lambda^\mathrm{A}, \lambda^\mathrm{H}$), and softmax temperature ($\lambda^{\mathrm{R}}$). Complete configurations are provided in our code repository at https://github.com/hxheart/BCR-DRL.

\textbf{\textit{Intrinsic Reward Truncation Epoch} ($N_{\mathrm{th}}$).}
This parameter in eq.~\eqref{eq_weights_final} governs the exploration-exploitation transition. The optimal value correlates with sparse reward improvement rates: rapidly increasing rewards indicate insufficient exploration, requiring later truncation, while plateaued rewards suggest earlier transition to exploitation. For \textit{Cramped Room}, we set $N_{\mathrm{th}} = 100$ based on the training curve inflection point (Fig.~\ref{fig_cr_sim_results1}), as Fig.~\ref{fig_ablation_CR} shows that $N_{\mathrm{th}} = 150$ leads to performance degradation from over-exploration. Following similar analysis, we set $N_{\mathrm{th}} = 150$ for \textit{Asymmetric Advantages} and $N_{\mathrm{th}} = 300$ for \textit{Coordination Ring}. This parameter is sensitive to environment-specific reward dynamics, as evidenced by the wide variation in optimal values across environments (100-300 epochs), requiring calibration rather than using a universal setting.

\textbf{\textit{Intrinsic Reward Coefficients} ($\lambda^\mathrm{A}$, $\lambda^\mathrm{H}$).}
These coefficients in eqs.~\eqref{eq_reward_effect_from_AIself} and~\eqref{eq_reward_effect_from_human} calibrate AI self-motivated and human-motivated intrinsic rewards. Optimal configuration maximizes exploration while maintaining effective exploitation. We maintain $\lambda^\mathrm{A} > \lambda^\mathrm{H}$ to compensate for lower human action frequency compared to AI actions, balancing both agents' behavioral influence. For \textit{Cramped Room} and \textit{Asymmetric Advantages}, empirical testing established $\lambda^\mathrm{A}=1$ and $\lambda^\mathrm{H}=0.02$. The spatially confined \textit{Coordination Ring}, which facilitates inherent exploration even in traditional PPO, benefits from reduced coefficients: $\lambda^\mathrm{A}=0.2$ and $\lambda^\mathrm{H}=0.01$. The coefficients show sensitivity to environment characteristics, particularly spatial constraints, though the relative ratio $\lambda^\mathrm{A} > \lambda^\mathrm{H}$ remains consistently effective across different scenarios.

\textbf{\textit{Softmax Temperature} ($\lambda^{\mathrm{R}}$):}
This parameter in eq.~\eqref{eq_weights} ensures smooth adaptation among all three reward components in eq.~\eqref{eq_reward_bcr_drl}. We determined $\lambda^{\mathrm{R}}=3$ to allow each component to maintain proportional influence throughout training. This parameter demonstrates low sensitivity to environment variations, with the uniform value across all tested environments indicating robust performance without environment-specific tuning.



\subsection{Generalization Study}\label{experiments_Exploration_env}
To evaluate the generalization capability of BCR-DRL beyond the Overcooked domain and the behavior-cloned human models used in previous experiments, we introduce a new environment, \textit{Exploration}, which features a distinct task setting and a different human-agent behavior model.

As shown in Fig.~\ref{fig_exploration_env}, this environment consists of a grid populated with unexplored cells (blank), explored cells (green), and obstacles (black). A human agent (blue) and an AI agent (red) jointly explore all accessible cells in the environment. Agents can choose from five discrete actions: move up, down, left, right, or stay. The environment is episodic, with each episode spanning 400 timesteps, potentially containing multiple rounds of exploration.
Unlike prior experiments that relied on human-cloned policies learned from real gameplay data, we employ a synthetic human model with a stochastic policy and reduced temporal frequency (acting every 10 timesteps) based on real human reaction time. This introduces significant behavioral uncertainty and temporal asynchrony, creating a coordination scenario that differs from the Overcooked setups. The AI agent, by contrast, acts at every timestep and must adaptively coordinate its exploration strategy to complement the human's delayed and unpredictable actions. Please also refer to Section~\ref{section_benchmark_HAIC} for the analysis of HAIC aspects of this \textit{Exploration} environment.
The reward structure incentivizes efficient and cooperative behavior: agents receive +2 for discovering a new cell, -0.5 for revisiting an already explored cell, and -1 for invalid actions (including boundary violations, collisions with obstacles or the other agent, or redundant actions with no position change). When all accessible cells are explored, both agents receive a +20 sparse reward, and the environment resets while retaining the obstacle layout.

This setup challenges the AI agent to anticipate human movements, avoid redundancy, and dynamically adapt its behavior under conditions of uncertainty and limited observability. It thereby provides a controlled but difficult testbed for evaluating whether BCR-DRL can generalize to novel coordination dynamics and behavior patterns that diverge from its training conditions.
Experimental results in Fig.~\ref{fig_exploration_experiment} show that BCR-DRL consistently outperforms the benchmark approaches (PPO$_\mathrm{HM}$ and Causal), achieving higher cumulative rewards and sample efficiency. These findings validate BCR-DRL's effectiveness in generalizing across environments and partner behavior models, reinforcing its potential as a robust solution for real-world HAIC applications.

\begin{figure}[!t]
    \centering
    \begin{subfigure}[t]{2.7cm}
        \includegraphics[width=2.4cm, trim=0 12 0 0, clip]{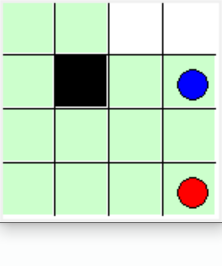} 
        \subcaption{\textit{Exploration} env.}
        \label{fig_exploration_env}
    \end{subfigure}
    \begin{subfigure}[t]{5.5cm}
        \includegraphics[width=5.2cm]{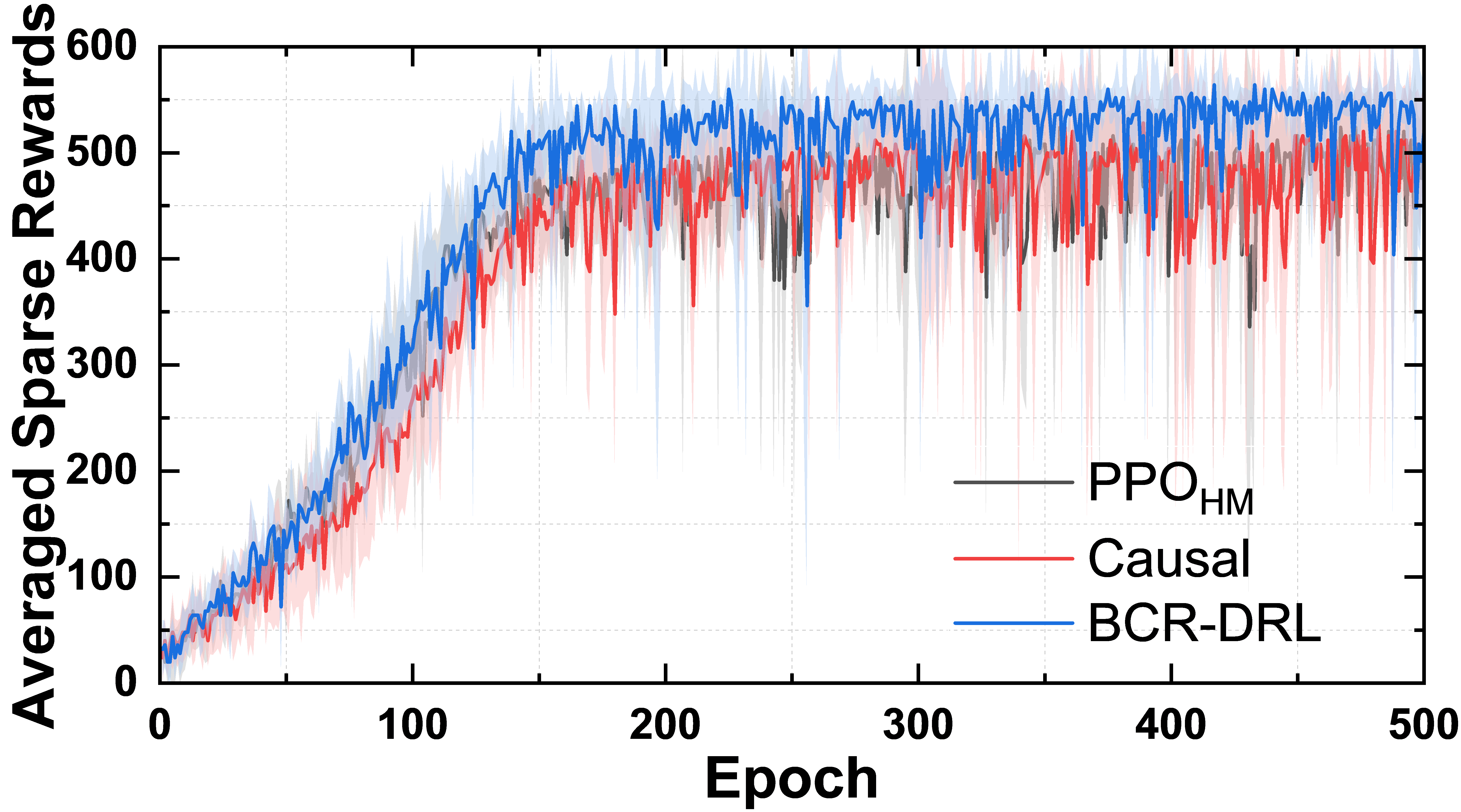}
        \subcaption{Averaged sparse rewards.}
        \label{fig_exploration_experiment}
    \end{subfigure}
    \vspace{10pt}
    \caption{Performance comparison based on BCR-DRL and the benchmark algorithms on the \textit{Exploration} environment.}
    \vspace{20pt}
\end{figure}

\section{Conclusion, Limitations, and Future Work}\label{section_conclusion}
We introduced a BCR-DRL to facilitate HAIC. Supplementing traditional extrinsic rewards, our BCR incorporated an innovative dual intrinsic rewarding scheme to facilitate comprehensive exploration and a novel context-aware weighting mechanism to optimize exploration and exploitation. Extensive experimental results across three layouts of the Overcooked environment demonstrated that BCR-DRL can increase rewards by approximately $20\%$ and improve sample efficiency by approximately $38\%$. Testing experiments underscored the algorithm's robustness, and the generalization study on the Exploration environment further validated BCR-DRL's generalization abilities across different coordination aspects of human and AI. 
Despite these, our approach has limitations. In HAIC environments with minimal exploration challenges, such as the \textit{Coordination Ring} layout, the benefits of BCR-DRL diminish. This suggests that the strength of BCR-DRL lies in tasks where uncertainty, complementarity, or partial observability pose significant coordination challenges. Furthermore, our current evaluation relies on simulated human models, this gap limits the immediate applicability of our findings to real-world HAIC scenarios. 
Future work will focus on deploying BCR-DRL in human-in-the-loop experiments. Based on our recently conducted preliminary real-human experiments, we have observed interesting phenomena that provide insights for potential readers. Specifically, we found that BCR-DRL coordinates more effectively with real humans compared to simulated models, with rewards increasing as humans become familiar with the environment due to their adaptability. In contrast, behavior cloning-based human models sometimes exhibit freezing behaviors since they lack adaptability. This will enable deeper insights into the practical implications and limitations of model-based reinforcement learning in HAIC, particularly regarding how human variability and adaptability impact coordination performance.



\bibliography{refs.bib}

\begin{thebibliography}{10}

\bibitem{NeurIPS2019_Overcooked_GitHub}
Micah Carroll, Rohin Shah, Mark~K. Ho, Thomas~L. Griffiths, Sanjit~A. Seshia, Pieter Abbeel, and Anca Dragan.
\newblock On the utility of learning about humans for human-{AI} coordination.
\newblock In {\em Proceedings of the 33rd International Conference on Neural Information Processing Systems}, Red Hook, NY, USA, 2019. Curran Associates Inc.

\bibitem{PantheonRL_aaai2022}
Bidipta Sarkar, Aditi Talati, Andy Shih, and Dorsa Sadigh.
\newblock Pantheonrl: A marl library for dynamic training interactions.
\newblock {\em Proceedings of the AAAI Conference on Artificial Intelligence}, 36(11):13221--13223, Jun. 2022.

\bibitem{Xin_BCDRL_TCOM}
Xin Hao, Phee~Lep Yeoh, Changyang She, Branka Vucetic, and Yonghui Li.
\newblock Secure deep reinforcement learning for dynamic resource allocation in wireless mec networks.
\newblock {\em IEEE Transactions on Communications}, 72(3):1414--1427, 2024.

\bibitem{Xin_BCDRL_ICC2024}
Xin Hao, Phee~Lep Yeoh, Changyang She, Yao Yu, Branka Vucetic, and Yonghui Li.
\newblock A constrained deep reinforcement learning optimization for reliable network slicing in a blockchain-secured low-latency wireless network.
\newblock In {\em ICC 2024 - IEEE International Conference on Communications}, pages 91--96, 2024.

\bibitem{AAAI_overcooked_Tsinghua}
Rui Zhao, Jinming Song, Yufeng Yuan, Haifeng Hu, Yang Gao, Yi~Wu, Zhongqian Sun, and Wei Yang.
\newblock Maximum entropy population-based training for zero-shot human-{AI} coordination.
\newblock {\em Proceedings of the AAAI Conference on Artificial Intelligence}, 37(5):6145--6153, Jun. 2023.

\bibitem{IJCAI2024_HRL_Overcooked}
Yi~Loo, Chen Gong, and Malika Meghjani.
\newblock A hierarchical approach to population training for human-{AI} collaboration.
\newblock In {\em Proceedings of the Thirty-Second International Joint Conference on Artificial Intelligence, {IJCAI-23}}, pages 3011--3019, 8 2023.

\bibitem{combine_human_aaai_2019}
Vinicius~G. Goecks, Gregory~M. Gremillion, Vernon~J. Lawhern, John Valasek, and Nicholas~R. Waytowich.
\newblock Efficiently combining human demonstrations and interventions for safe training of autonomous systems in real-time.
\newblock {\em Proceedings of the AAAI Conference on Artificial Intelligence}, 33(01):2462--2470, Jul. 2019.

\bibitem{multiple_feedback_aaai_2023}
Gaurav~R. Ghosal, Matthew Zurek, Daniel~S. Brown, and Anca~D. Dragan.
\newblock The effect of modeling human rationality level on learning rewards from multiple feedback types.
\newblock In {\em Proceedings of the Thirty-Seventh AAAI Conference on Artificial Intelligence}, AAAI'23. AAAI Press, 2023.

\bibitem{balance_human_aaai_2020}
Tianxiang Zhao, Lemao Liu, Guoping Huang, Huayang Li, Yingling Liu, Liu GuiQuan, and Shuming Shi.
\newblock Balancing quality and human involvement: An effective approach to interactive neural machine translation.
\newblock {\em Proceedings of the AAAI Conference on Artificial Intelligence}, 34(05):9660--9667, Apr. 2020.

\bibitem{ICML2019_Social_MARL_GitHub}
Natasha Jaques, Angeliki Lazaridou, Edward Hughes, Çaglar G{\"u}lçehre, Pedro~A. Ortega, DJ~Strouse, Joel~Z. Leibo, and Nando de~Freitas.
\newblock Social influence as intrinsic motivation for multi-agent deep reinforcement learning.
\newblock In {\em Proceedings of the International Conference on Machine Learning}, 2019.

\bibitem{max_entropy_reward}
Tuomas Haarnoja, Aurick Zhou, Pieter Abbeel, and Sergey Levine.
\newblock Soft actor-critic: Off-policy maximum entropy deep reinforcement learning with a stochastic actor.
\newblock In {\em Proceedings of the 35th International Conference on Machine Learning}, volume~80, pages 1861--1870, 10--15 Jul 2018.

\bibitem{NIPS2017_MAAC}
Ryan Lowe, YI~WU, Aviv Tamar, Jean Harb, OpenAI Pieter~Abbeel, and Igor Mordatch.
\newblock Multi-agent actor-critic for mixed cooperative-competitive environments.
\newblock In I.~Guyon, U.~Von Luxburg, S.~Bengio, H.~Wallach, R.~Fergus, S.~Vishwanathan, and R.~Garnett, editors, {\em Advances in Neural Information Processing Systems}, volume~30. Curran Associates, Inc., 2017.

\bibitem{Cooperative_IRL}
Dylan Hadfield-Menell, Anca Dragan, Pieter Abbeel, and Stuart Russell.
\newblock Cooperative inverse reinforcement learning.
\newblock In {\em Proceedings of the 30th International Conference on Neural Information Processing Systems}, NIPS'16, page 3916–3924, Red Hook, NY, USA, 2016. Curran Associates Inc.

\bibitem{ex_ex_self_supervised}
Deepak Pathak, Pulkit Agrawal, Alexei~A. Efros, and Trevor Darrell.
\newblock Curiosity-driven exploration by self-supervised prediction.
\newblock In {\em Proceedings of the 34th International Conference on Machine Learning - Volume 70}, ICML'17, page 2778–2787. JMLR.org, 2017.

\bibitem{Sutton_FlowerBook}
Richard~S. Sutton and Andrew~G. Barto.
\newblock {\em Reinforcement Learning: An Introduction}.
\newblock The MIT Press, second edition, 2018.

\bibitem{classical_A3C}
Volodymyr Mnih, Adria~Puigdomenech Badia, Mehdi Mirza, Alex Graves, Timothy Lillicrap, Tim Harley, David Silver, and Koray Kavukcuoglu.
\newblock Asynchronous methods for deep reinforcement learning.
\newblock In {\em Proceedings of The 33rd International Conference on Machine Learning}, volume~48, pages 1928--1937, New York, USA, 20--22 Jun 2016. PMLR.

\bibitem{ex_ex_MARL}
Tabish Rashid, Mikayel Samvelyan, Christian~Schroeder De~Witt, Gregory Farquhar, Jakob Foerster, and Shimon Whiteson.
\newblock Monotonic value function factorisation for deep multi-agent reinforcement learning.
\newblock {\em J. Mach. Learn. Res.}, 21(1), jan 2020.

\bibitem{CTDE}
Christopher Amato.
\newblock An introduction to centralized training for decentralized execution in cooperative multi-agent reinforcement learning, 2024.

\bibitem{NoFreeLunch}
D.H. Wolpert and W.G. Macready.
\newblock No free lunch theorems for optimization.
\newblock {\em IEEE Transactions on Evolutionary Computation}, 1(1):67--82, 1997.

\bibitem{Xin_HML}
Xin Hao, Changyang She, Phee Lep~Yeoh, Yuhong Liu, Branka Vucetic, and Yonghui Li.
\newblock Hybrid-task meta-learning: A {GNN} approach for scalable and transferable bandwidth allocation.
\newblock {\em IEEE Transactions on Wireless Communications}, 23(12):19820--19835, 2024.

\bibitem{CVPR2019_MultiTask_attention}
Shikun Liu, Edward Johns, and Andrew~J. Davison.
\newblock End-to-end multi-task learning with attention.
\newblock In {\em 2019 IEEE/CVF Conference on Computer Vision and Pattern Recognition (CVPR)}, pages 1871--1880, 2019.

\bibitem{code_github_Overcooked}
Micah Carroll.
\newblock {HumanCompatibleAI}, {overcooked-ai}.
\newblock \url{https://github.com/HumanCompatibleAI/overcooked_ai}, 2024.
\newblock Accessed: 2024-11-13.

\bibitem{Harvard_OfflineRL}
Zana Buçinca, Siddharth Swaroop, Amanda~E. Paluch, Susan~A. Murphy, and Krzysztof~Z. Gajos.
\newblock Towards optimizing human-centric objectives in {AI}-assisted decision-making with offline reinforcement learning.
\newblock {\em arXiv preprint arXiv:2403.05911}, 2024.

\bibitem{reward_shaping_nips2020}
Yujing Hu, Weixun Wang, Hangtian Jia, Yixiang Wang, Yingfeng Chen, Jianye Hao, Feng Wu, and Changjie Fan.
\newblock Learning to utilize shaping rewards: A new approach of reward shaping.
\newblock In {\em Proceedings of the 34th International Conference on Neural Information Processing Systems}, NIPS '20, Red Hook, NY, USA, 2020. Curran Associates Inc.

\bibitem{reward_shaping_WED}
Andrew~Y. Ng, Daishi Harada, and Stuart~J. Russell.
\newblock Policy invariance under reward transformations: Theory and application to reward shaping.
\newblock In {\em Proceedings of the Sixteenth International Conference on Machine Learning}, ICML '99, page 278–287, San Francisco, CA, USA, 1999.

\bibitem{ppo_classic}
John Schulman, Filip Wolski, Prafulla Dhariwal, Alec Radford, and Oleg Klimov.
\newblock Proximal policy optimization algorithms.
\newblock {\em CoRR}, abs/1707.06347, 2017.

\bibitem{code_keras_ppo}
Ilias Chrysovergis.
\newblock Code examples, proximal policy optimization.
\newblock \url{https://keras.io/examples/rl/ppo_cartpole/}, 2024.
\newblock Accessed: 2024-11-13.

\end{thebibliography}

\end{document}